\documentclass{article} 
\usepackage{iclr2024_conference,times}

\usepackage{amsmath,amsfonts,bm}

\def\eqref#1{equation~\ref{#1}}

\def\1{\bm{1}}

\DeclareMathAlphabet{\mathsfit}{\encodingdefault}{\sfdefault}{m}{sl}
\SetMathAlphabet{\mathsfit}{bold}{\encodingdefault}{\sfdefault}{bx}{n}

\usepackage{hyperref}
\usepackage{url}

\usepackage{wrapfig}
\usepackage{graphicx}
\usepackage[utf8]{inputenc} 
\usepackage[T1]{fontenc}    
\usepackage{hyperref}       
\usepackage{url}            
\usepackage{booktabs}       
\usepackage{amsfonts}       
\usepackage{nicefrac}       
\usepackage{microtype}      
\usepackage{xcolor}         

\usepackage{amsmath,amssymb,amsthm}
\usepackage{thmtools}
\usepackage{thm-restate}
\usepackage{cleveref}
\usepackage{enumitem}

\usepackage{amsmath,amssymb,amsthm}
\usepackage{thmtools}

\usepackage{bbm}
\usepackage{thm-restate}
\usepackage{enumitem}
\usepackage{hyperref}
\usepackage{cleveref}
\usepackage{dblfloatfix}

\usepackage{subcaption}
\usepackage{graphicx}
\usepackage{float}

\newif\ifcomments
 \commentstrue
\ifcomments
\newcommand\emma[1]{\textcolor{blue}{\textbf{EP}: #1}}
\newcommand\nikhil[1]{\textcolor{teal}{\textbf{NG}: #1}}

\else
    \newcommand\emma[1]{}
    \newcommand\nikhil[1]{}
\fi

\newcommand\bbE{\ensuremath{\mathbb{E}}}
\newcommand\betaY{\ensuremath{\boldsymbol{\beta_Y}}}
\newcommand\betadelta{\ensuremath{\boldsymbol{\beta_\Delta}}}
\newcommand\betaT{\ensuremath{\boldsymbol{\beta_T}}}

\newcommand\thetacon{\ensuremath{\theta_{\text{con}}}}
\newcommand\thetaunc{\ensuremath{\theta_{\text{unc}}}}
\newcommand\capXi{\ensuremath{X_i}}
\newcommand\capX{\ensuremath{X}}

\newcommand\numfeatures{\ensuremath{D}}
\newcommand\tildebetaT{\ensuremath{\tilde{\boldsymbol{\beta}}_{\boldsymbol{T}}}}
\newcommand\tildebetaY{\ensuremath{\tilde{\boldsymbol{\beta}}_{\boldsymbol{Y}}}}

\hypersetup{
    colorlinks,
    linkcolor={red!50!black},
    citecolor={blue!50!black},
    urlcolor={blue!80!black}
}

\title{Domain constraints improve risk prediction when outcome data is missing}

\author{%
Sidhika Balachandar~\thanks{Correspondence to: \texttt{sidhikab@cs.cornell.edu}} \\
  Cornell Tech
  \And
  Nikhil Garg \\
  Cornell Tech
  \And
  Emma Pierson \\
  Cornell Tech
}

\iclrfinalcopy 
\begin{document}

\maketitle

\begin{abstract}
Machine learning models are often trained to predict the outcome resulting from a human decision. For example, if a doctor decides to test a patient for disease, will the patient test positive? A challenge is that historical decision-making determines whether the outcome is observed: we only observe test outcomes for patients doctors historically tested. Untested patients, for whom outcomes are unobserved, may differ from tested patients along observed and unobserved dimensions. We propose a Bayesian model class which captures this setting. The purpose of the model is to accurately estimate risk for both tested and untested patients. Estimating this model is challenging due to the wide range of possibilities for untested patients. To address this, we propose two domain constraints which are plausible in health settings: a \emph{prevalence constraint}, where the overall disease prevalence is known, and an \emph{expertise constraint}, where the human decision-maker deviates from purely risk-based decision-making only along a constrained feature set. We show theoretically and on synthetic data that domain constraints improve parameter inference. We apply our model to a case study of cancer risk prediction, showing that the model's inferred risk predicts cancer diagnoses, its inferred testing policy captures known public health policies, and it can identify suboptimalities in test allocation. Though our case study is in healthcare, our analysis reveals a general class of domain constraints which can improve model estimation in many settings.
\end{abstract}

\section{Introduction}
Machine learning models are often trained to predict outcomes in settings where a human makes a high-stakes decision. In criminal justice, a judge decides whether to release a defendant prior to trial, and models are trained to predict whether the defendant will fail to appear or commit a crime if released~\citep{lakkaraju2017selective,jung2020simple,kleinberg2018human}. In lending, a creditor decides whether to grant an applicant a loan, and models are trained to predict whether the applicant will repay~\citep{bjorkegren2020behavior,crook2004does}. In healthcare---the setting we focus on in this paper---a doctor decides whether to test a patient for disease, and models are trained to predict whether the patient will test positive~\citep{jehi2020individualizing,mcdonald2021derivation,mullainathan2022diagnosing}. 
Machine learning predictions help guide decision-making in all these settings. A model which predicts a patient's risk of disease can help allocate tests to the highest-risk patients, and also identify suboptimalities in human decision-making: for example, testing patients at low risk of disease, or failing to test high risk patients~\citep{mullainathan2022diagnosing}. 

A fundamental challenge in all these settings is that historical decision-making determines whether the outcome is observed. In criminal justice, release outcomes are only observed for defendants judges have historically released. In lending, loan repayments are only observed for applicants historically granted loans. In healthcare, test outcomes are only observed for patients doctors have historically tested. This is problematic because the model must make accurate predictions for the entire population, not just the historically tested population. Learning only from the tested population also risks introducing bias against underserved populations who are less likely to get medical tests partly due to worse healthcare access~\citep{chen2021ethical,pierson2020assessing,servik2020huge,jain2023antiracist}. Thus, there is a challenging distribution shift between the tested and untested populations. The two populations may differ both along \emph{observables} recorded in the data and \emph{unobservables} known to the human decision-maker but unrecorded in the data. For example, tested patients may have more symptoms recorded than untested patients---but they may also differ on unobservables, like how much pain they are in or how sick they look, which are known to the doctor but are not available for the model. 
This setting, referred to as the \emph{selective labels} setting~\citep{lakkaraju2017selective}, occurs in high-stakes domains including healthcare, hiring, insurance, lending, education, welfare services, government inspections, tax auditing, recommender systems, wildlife protection, and criminal justice and has been the subject of substantial academic interest (see \S \ref{sec:related_work} for related work).

Without further constraints on the data generating process, there is a wide range of possibilities for the untested patients. They could all have the disease or never have the disease. However, selective labels settings often have \emph{domain-specific constraints} which would allow us to limit the range of possibilities. For example, in medical settings, we might know the prevalence of a disease in the population. Recent distribution shift literature has shown that generic methods generally do not perform well across all distribution shifts and that domain-specific constraints can improve generalization~\citep{gulrajani2020search,koh2021wilds,sagawa2021extending,gao2023out,kaur2022modeling,tellez2019quantifying,wiles2021fine}. This suggests the utility of domain constraints in improving generalization from the tested to untested population.

Motivated by this reasoning, we make the following contributions: 
\begin{enumerate}
\item We propose a Bayesian model class which captures the selective labels setting and nests classic econometric models. We model a patient's risk of disease as a function of observables and unobservables. The probability of testing a patient increases with disease risk and other factors (e.g., bias). The purpose of the model is to accurately estimate risk for both the tested and untested patients and to quantify deviations from purely risk-based test allocation. 
\item We propose two constraints informed by the medical domain to improve model estimation: a \emph{prevalence constraint}, where disease prevalence is known, and an \emph{expertise constraint}, where the decision-maker deviates from risk-based decision-making along a constrained feature set. We show theoretically and on synthetic data that the constraints improve inference.
\item We apply our model to a breast cancer risk prediction case study. We conduct a suite of validations, showing that the model's
(i) inferred risks predict cancer diagnoses, (ii) inferred unobservables correlate with known unobservables, (iii) inferred predictors of cancer risk correlate with known predictors, and (iv) inferred testing policy correlates with public health policies. We also show that our model identifies deviations from risk-based test allocation and that the prevalence constraint increases the plausibility of inferences. 
\end{enumerate} 

Though our case study is in healthcare, our analysis reveals a general class of domain constraints which can improve model estimation in many selective labels settings. 
\section
{Model}
\label{sec:model}
We now describe our Bayesian model class. Following previous work~\citep{mullainathan2022diagnosing}, our underlying assumption is that whether a patient is tested for a disease should be determined primarily by their risk of disease. Thus, the purpose of the model is to accurately estimate risk for both the tested and untested patients and to quantify deviations from purely risk-based test allocation. The latter task relates to literature on diagnosing factors affecting human decision-making~\citep{mullainathan2022diagnosing,zamfirescu2022trucks,jung2018omitted}. 

Consider a set of people indexed by $i$. For each person, we see observed features $\capXi\in\mathbb{R}^\numfeatures$ (e.g., demographics and symptoms in an electronic health record). We observe a \textit{testing decision} $T_i\in\{0,1\}$, where $T_i=1$ indicates that the $i$th person was tested. If the person was tested $(T_i=1)$, we observe an outcome $Y_i$. $Y_i$ might be a binary indicator (e.g. $Y_i=1$ means that the person tests positive), or $Y_i$ might be a numeric outcome of a medical test (e.g. T cell count or oxygen saturation levels). Throughout, we generally refer to $Y_i$ as a binary indicator, but our framework extends to non-binary $Y_i$, and we derive our theoretical results in this setting. If $T_i=0$ we do not observe $Y_i$.

There are \emph{unobservables}~\citep{angrist2009mostly, rambachan2022counterfactual}, denoted by $Z_i \in\mathbb{R}$, that affect \textit{both} $T_i$ and $Y_i$ but are not recorded in the dataset -- e.g., whether the doctor observes that the person is in pain. Consequently, the risk of the tested population differs from the untested population even conditional on observables $\capXi$: i.e. $p(Y_i | T_i=1, \capXi) \neq p(Y_i | T_i=0, \capXi)$. 

A person's risk of disease is captured by their \emph{risk score} $r_i \in \mathbb{R}$, which is a function of $X_i$ and $Z_i$. Whether the person is tested ($T_i = 1$) depends on their risk score $r_i$, but also factors like screening policies or socioeconomic disparities. More formally, our data generating process is
\begin{align}
\label{eq:DGP}
\begin{split}
\text{Unobservables:} & \quad Z_i \sim f(\cdot|\sigma^2)\\
\text{Risk score:} & \quad r_i = \capXi^T\betaY+Z_i\\
\text{Test outcome:} & \quad Y_i\sim h_Y(\cdot | r_i)\\
\text{Testing decision:} & \quad T_i\sim h_T(\cdot | \alpha r_i+\capXi^T\betadelta)\, .
\end{split}
\end{align}
\begin{wrapfigure}{r}{0.4\textwidth}
\vspace{-2.5em}
  \begin{center}
    \includegraphics[width=0.38\textwidth]{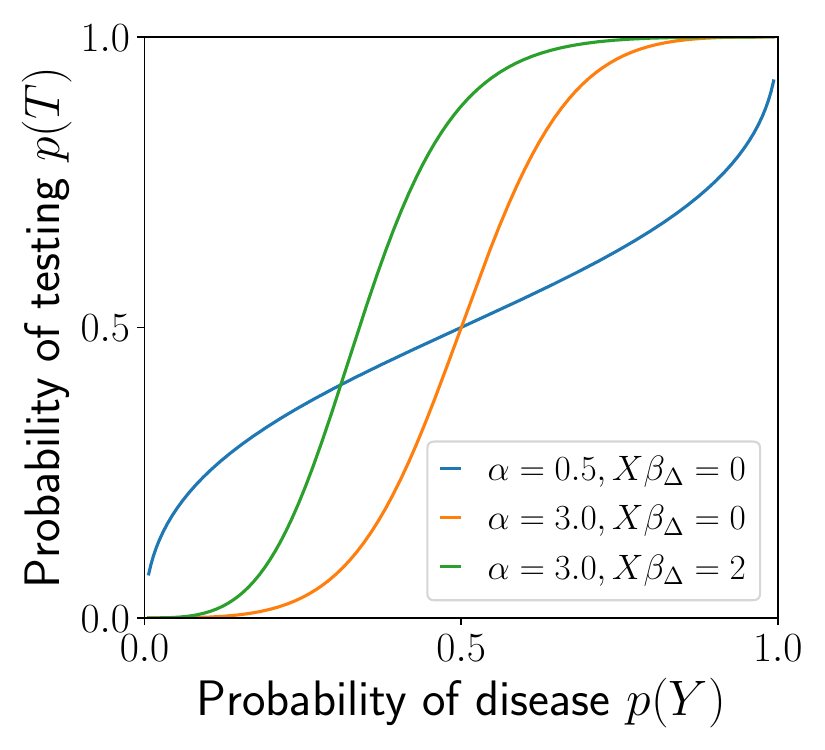}
  \end{center}
  \caption{Effect of $\alpha$ and $\capX\betadelta$: $\alpha$ controls how steeply testing probability $p(T_i)$ increases in disease risk $p(Y_i)$, while $\capX\betadelta$ captures factors which affect $p(T_i)$ when controlling for $p(Y_i)$.}
  \label{fig:illustrative_fig}
  \vspace{-2em}
\end{wrapfigure}
In words, $Z_i$ is drawn from a distribution $f$ with parameter $\sigma^2$, which captures the relative importance of the unobserved versus observed features. The disease risk score $r_i\in \mathbb{R}$ is modeled as a linear function of observed features (with unknown coefficients $\betaY\in\mathbb{R}^\numfeatures$) and the unobserved $Z_i$. $Y_i$ is drawn from a distribution $h_Y$ parameterized by $r_i$ -- e.g., $Y_i\sim \textrm{Bernoulli}(\text{sigmoid}(r_i))$. Analogously, the testing decision $T_i$ is drawn from a distribution $h_T$ parameterized by a linear function of the true disease risk score and other factors, with unknown coefficients $\alpha\in\mathbb{R}$ and $\betadelta\in\mathbb{R}^\numfeatures$. Because $T_i$ depends on $r_i$, and $r_i$ is a function of $Z_i$, $T_i$ depends on $Z_i$. Figure \ref{fig:illustrative_fig} illustrates the effect of $\alpha$ and  $\betadelta$. A larger $\alpha$ indicates that testing probability increases more steeply in risk. $\betadelta$ captures human or policy factors which affect a patient's probability of being tested beyond disease risk. In other words, $\betadelta$ captures deviations from purely risk-based test allocation. Putting things together, the model parameters are $\theta \triangleq (\alpha, \sigma^2, \betadelta, \betaY)$. 
\paragraph{Medical domain knowledge:} Besides the observed data, in medical settings we often have constraints to aid model estimation. We consider two constraints.
\begin{itemize}[leftmargin=2em]
\item \textbf{Prevalence constraint:} The average value of $Y$ across the entire population is known ($\mathbb{E}[Y]$). When $Y$ is a binary indicator of whether a patient has a disease, this corresponds to assuming that the \emph{disease prevalence} is known. This assumption is plausible because estimating prevalence has been the focus of substantial public health research, and estimates thus exist in many medical settings; for more details see appendix~\ref{sec:appendix_prevalence_constraint}. For example, this information is available for cancer~\citep{cr_prevalence_stats}, COVID-19~\citep{covid_19_prevalence}, and heart disease~\citep{centers2007prevalence}. In some cases, the prevalence is only \emph{approximately} known~\citep{manski2021estimating, manski2020bounding, mullahy2021embracing}; our Bayesian formulation can incorporate such soft constraints as well. 
\item \textbf{Expertise constraint:} Because doctors and patients are informed decision-makers, we can assume that tests are allocated \textit{mostly} based on disease risk. Specifically, we assume that there are some features which do not affect a patient's probability of receiving a test when controlling for their risk: i.e., that ${\betadelta}_{d} = 0$, for at least one dimension $d$. For example, we may assume that when controlling for disease risk, a patient's height does not affect their probability of being tested for cancer, and thus ${\betadelta}_{\text{height}}=0$.
\end{itemize}
\section{Theoretical Analysis}
\label{sec:theory}
In this section, we prove why our proposed constraints improve parameter inference by analyzing a special case of our general model in \eqref{eq:DGP}. In Proposition \ref{heckman}, we show that this special case is equivalent to the Heckman model~\citep{heckman1976common,heckman1979sample}, which is used to correct bias from non-randomly selected samples. In Proposition \ref{identifiability}, we analyze this model to show that constraints can improve the precision of parameter inference. The full proofs are in Appendix \ref{sec:appendix_proofs}. In Sections \ref{sec:synthetic_experiments} and \ref{sec:real_data_experiments} we empirically generalize our theoretical results beyond the special Heckman case.
\subsection{Domain constraints can improve the precision of parameter inference}
\label{sec:heckman_model}
We start by defining the Heckman model and showing it is a special case of our general model.
\begin{restatable}[Heckman correction model]{definition}{heckmandef}
\label{heckman_model}
The Heckman model can be written in the following form~\citep{heckman_model}:
\begin{align}
\begin{split}
 T_i &= \mathbbm{1}[\capXi^T\tildebetaT + u_i > 0]\\
 Y_i &= \capXi^T\tildebetaY + Z_i\\
  \begin{bmatrix}
           u_i \\
           Z_i
         \end{bmatrix} &\sim \text{Normal}\bigg(\begin{bmatrix}
           0 \\
           0
         \end{bmatrix}, 
         \begin{bmatrix}
           1&\tilde{\rho} \\
           \tilde{\rho}&\tilde{\sigma}^2
         \end{bmatrix}
         \bigg)\, .
\end{split}
 \label{eq:properheckmanmodel}
\end{align}
\end{restatable}
\begin{restatable}[]{proposition}{heckman}
\label{heckman}
The Heckman model (Definition \ref{heckman_model}) is equivalent to the following special case of the general model in \eqref{eq:DGP}:
\begin{align}
\begin{split}
Z_i &\sim \mathcal
{N}(0, \sigma^2)\\
r_i &= \capXi^T\betaY+Z_i\\
Y_i&=r_i\\
T_i&\sim \text{Bernoulli}(\Phi(\alpha r_i+\capXi^T\betadelta))\, .
\end{split}
\label{eq:heckmanmodel}
\end{align}
\end{restatable}
It is known that the Heckman model is identifiable~\citep{lewbel2019identification}, and thus the special case of our model is identifiable (i.e., distinct parameter sets correspond to distinct observed expectations) without further constraints. However, past work has often placed constraints on the Heckman model (though different constraints from those we propose) to improve parameter inference. Without constraints, the model is only weakly identified by functional form assumptions~\citep{lewbel2019identification}. This suggests that our proposed constraints could also improve model estimation. In Proposition \ref{identifiability}, we make this intuition precise by showing that our proposed constraints improve the \textit{precision} of the parameter estimates as measured by the \textit{variance} of the parameter posteriors. 

In our Bayesian formulation, we estimate a posterior distribution for parameter $\theta$ given the observed data: $g(\theta) \triangleq p(\theta|\capX, T, Y)$. Let $\textrm{Var}(\theta)$ denote the variance of $g(\theta)$. We show that constraining the value of any one parameter \textit{will not worsen} the precision with which other parameters are inferred. In particular, constraining a parameter $\thetacon$ to a value drawn from its posterior distribution will not in expectation increase the posterior variance of any other unconstrained parameters $\thetaunc$. To formalize this, we define the \emph{expected conditional variance}:
\begin{restatable}[Expected conditional variance]{definition}{conditionalvariance}
\label{conditionalvariance}
Let the distribution over model parameters $g(\theta) \triangleq p(\theta|\capX, T, Y)$ be the posterior distribution of the parameters $\theta$ given the observed data $\{\capX, T, Y\}$. We define the expected conditional variance of an unconstrained parameter $\thetaunc$, conditioned on the value of a constrained parameter $\thetacon$, to be $\mathbb{E}[\textrm{Var}(\thetaunc|\thetacon)] \triangleq \mathbb{E}_{\thetacon^* \sim g}[\textrm{Var}(\thetaunc|\thetacon=\thetacon^*)]$.
\end{restatable}
\begin{restatable}[]{proposition}{identifiability}
\label{identifiability}
In expectation, constraining the parameter $\thetacon$ does not increase the variance of any other parameter $\thetaunc$. In other words, $\mathbb{E}[\textrm{Var}(\thetaunc|\thetacon)] \leq \textrm{Var}(\thetaunc)$. Moreover, the inequality is strict as long as $\mathbb{E}[\thetaunc|\thetacon]$ is non-constant in $\thetacon$ (i.e., $\textrm{Var}(\bbE[\thetaunc|\thetacon])>0$).
\end{restatable} 
In other words, we reason about the effects of fixing a parameter $\thetacon$ to its true value $\thetacon^*$. That value $\thetacon^*$ is distributed according to the posterior distribution $g$, and so we reason about expectations over $g$. In expectation, fixing the value of $\thetacon$ does not increase the variance of any other parameter $\thetaunc$, and strictly reduces it as long as the expectation of $\thetaunc$ is non-constant in $\thetacon$. 

Both the expertise and prevalence constraints fix the value of at least one parameter. The expertise constraint fixes the value of ${\betadelta}_d$ for some $d$. For the Heckman model, the prevalence constraint fixes the value of the intercept ${\betaY}_0$ (assuming the standard condition that columns of $X$ are zero-mean except for an intercept column of ones). Thus, Proposition \ref{identifiability} implies that both constraints will not increase the variance of other model parameters, and will strictly reduce it as long as the posterior expectations of the unconstrained parameters are non-constant in the constrained parameters. In Appendix \ref{sec:appendix_proofs} we prove Proposition \ref{identifiability} and provide conditions under which the constraints strictly reduce the variance of other model parameters. We also verify and extend these theoretical results on synthetic data (Appendix \ref{sec:appendix_heckman_model} Figure \ref{fig:heckman_plot}).
\subsection{Empirical extension beyond the Heckman special case} 
\label{sec:empirical_extension_beyond_linear_case}
While we derive our theoretical results for a special case of our general model class, in our experiments (\S \ref{sec:synthetic_experiments} and \S \ref{sec:real_data_experiments}) we validate they hold beyond this special case by using a Bernoulli-sigmoid model: 
\begin{align}
\label{eq:uniform_model}
\begin{split}
Z_i &\sim \text{Uniform}(0, \sigma^2)\\
r_i &= \capXi^T\betaY+Z_i\\
Y_i&\sim \text{Bernoulli}(\text{sigmoid}(r_i))\\
T_i&\sim \text{Bernoulli}(\text{sigmoid}(\alpha r_i+\capXi^T\betadelta))\, .
\end{split}
\end{align}
We note two ways in which this model differs from the Heckman model. First, it uses a \textit{binary} disease outcome $Y$ because this is an appropriate choice for our breast cancer case study (\S \ref{sec:real_data_experiments}). With a binary outcome, models are known to be more challenging to fit: one cannot simultaneously estimate $\alpha$ and $\sigma$, and models fit without constraints may fail to recover the correct parameters~\citep{stata, vandeven1981demand,toomet2008sample}. Even in this more challenging case, we show that our proposed constraints improve model estimation. Second, this model uses a uniform distribution of unobservables instead of a normal distribution of unobservables. As we show in Appendix \ref{sec:appendix_uniform_model}, this choice allows us to marginalize out $Z_i$, greatly accelerating model-fitting.
\section{Synthetic experiments}
\label{sec:synthetic_experiments}
We now validate our proposed approach on synthetic data. Our theoretical results imply that our proposed constraints should reduce the variance of parameter posteriors (improving precision). We verify that this is the case. We also show empirically that the proposed constraints produce posterior mean estimates which lie closer to the true parameter values (improving accuracy).
\begin{figure}
\vspace{-2em}
  \begin{center}
    \includegraphics[width=0.9\textwidth]{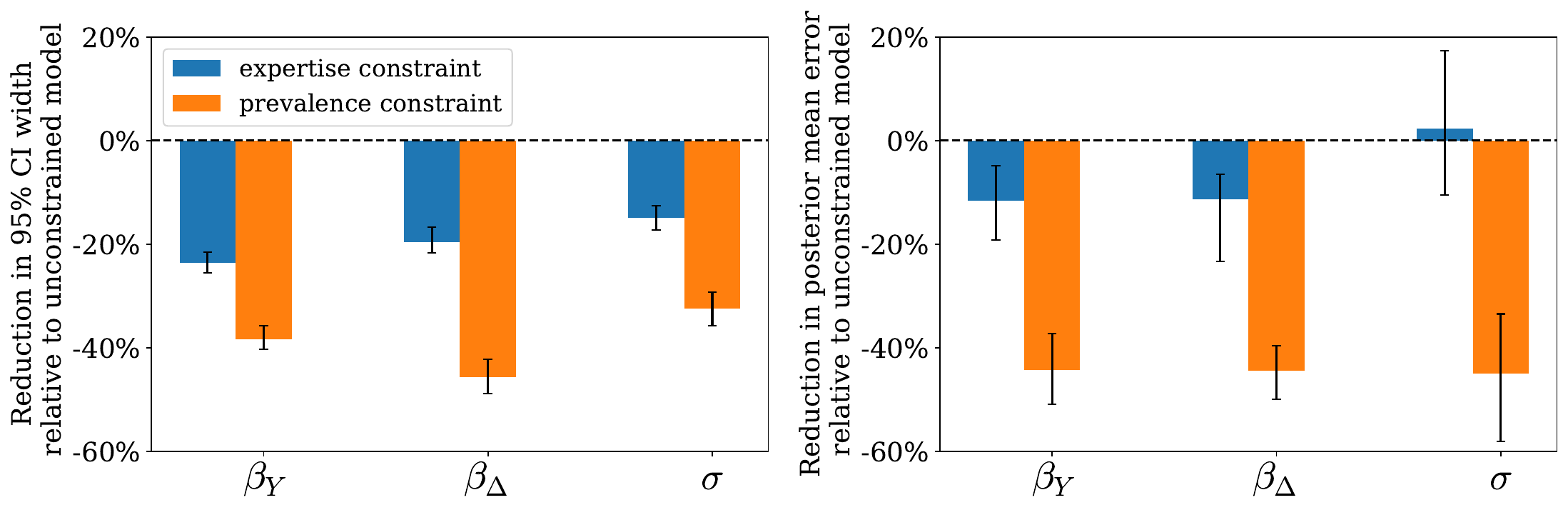}
  \end{center}
  \caption{The prevalence and expertise constraints each produce more precise and accurate inferences on synthetic data drawn from the Bernoulli-sigmoid model with uniform noise (\eqref{eq:uniform_model}). To quantify precision (left), we report the percent reduction in 95\% confidence interval width as compared to the unconstrained model. To quantify accuracy (right), we report the percent reduction in posterior mean error --- i.e., the absolute difference between  the posterior mean and the true parameter value --- as compared to the unconstrained model. We plot the median across 200 synthetic datasets. Error bars denote the bootstrapped 95\% confidence interval on the median.}
  \label{fig:bar_plot}
  \vspace{-1.5em}
\end{figure}

In Appendix \ref{sec:appendix_heckman_model}, we show experimentally that these results hold for the Heckman special case of our general model.
Here we show that our theoretical results apply beyond the Heckman special case by conducting experiments on models with binary outcomes and multiple noise distributions. For all experiments, we use the Bayesian inference package Stan~\citep{carpenter2017stan}, which uses the Hamiltonian Monte Carlo algorithm~\citep{betancourt2017conceptual}. We report results across 200 trials. For each trial, we generate a new dataset from the data generating process the model assumes; fit the model to that dataset; and evaluate model fit using two metrics: \emph{precision} (width of the 95\% confidence interval) and \emph{accuracy} (difference between the posterior mean and the true parameter value). We wish to assess the effect of the constraints on model inferences. Thus, we compare inferences from models with: (i) no constraints (unconstrained); (ii) a prevalence constraint; and (iii) an expertise constraint on a subset of the features. Details are in Appendix \ref{sec:appendix_syn_experiments} and the code is at \url{https://github.com/sidhikabalachandar/domain_constraints}.
 
Figure \ref{fig:bar_plot} shows results for the Bernoulli-sigmoid model with uniform unobservables (\eqref{eq:uniform_model}). Both constraints generally produce more precise and accurate inferences for all parameters relative to the unconstrained model. The one exception is that the expertise constraint does not improve accuracy for $\sigma^2$. Overall, the synthetic experiments corroborate and extend the theoretical analysis, showing that the proposed constraints improve precision and accuracy of parameter estimates for several variants of our general model. (In Appendix \ref{sec:appendix_syn_experiments}, we also provide results for other variants of our general model, including alternate distributions of unobservables (Figures \ref{fig:normal_fixed_sigma} and \ref{fig:normal_fixed_alpha}); higher-dimensional features (Figure \ref{fig:quadruple_features}); and non-linear interactions between features (Figure \ref{fig:pairwise_interactions}).)
\section{Real-world case study: Breast cancer testing}
\label{sec:real_data_experiments}
To demonstrate our model's applicability to healthcare settings, we apply it to a breast cancer testing dataset. In this setting, $\capXi$ consists of features capturing the person's demographics, genetics, and medical history; $T_i \in \{0, 1\}$ denotes whether a person has been tested for breast cancer; and $Y_i \in \{0, 1\}$ denotes whether the person is diagnosed with breast cancer. Our goal is to learn each person's risk of cancer---i.e., $p(Y_i=1|X_i)$. We focus on a younger population (age $\leq$ 45) because it creates a challenging distribution shift between the tested and untested populations. Younger people are generally not tested for cancer~\citep{nhs_screening_doc}, so the tested population $(T_i=1)$ may differ from the untested population, including on unobservables.

In the following sections, we describe our experimental set up and the model we fit (\S \ref{sec:experimental_setup}), we conduct four validations on the fitted model (\S \ref{sec:validating_model}), we use the model to assess historical testing decisions (\S \ref{sec:human_decisions}), and we compare to a model fit without a prevalence constraint (\S \ref{sec:comparison_without_prevalence}).
\subsection{Experimental setup}
\label{sec:experimental_setup}
Our data comes from the UK Biobank~\citep{sudlow2015uk}, which contains information on health, demographics, and genetics for the UK (see Appendix \ref{sec:appendix_ukbiobank} for details). We analyze 54,746 people by filtering for women under the age of 45 (there is no data on breast cancer tests for men). For each person, $X_i$ consists of 7 health, demographic, and genetic features found to be predictive of breast cancer~\citep{nih_risk_tool, cancer_risk_predictors2,yanes2020clinical}. $T_i \in \{0, 1\}$ denotes whether the person receives a mammogram (the most common breast cancer test) in the 10 years following measurement of features. $Y_i \in \{0, 1\}$ denotes whether the person is diagnosed with breast cancer in the 10 year period. $p(T=1) = 0.51$ and $p(Y=1|T=1)=0.03$.\footnote{We verify that very few people in the dataset have $T = 0$ and $Y = 1$ (i.e., are diagnosed with no record of a test): $p(Y=1|T=0) = 0.0005$. We group these people with the untested $T=0$ population, since they did not receive a breast cancer test.} 

As in the synthetic experiments, we fit the Bernoulli-sigmoid model with uniform unobservables (\eqref{eq:uniform_model}). We include a prevalence constraint $\mathbb{E}[Y] = 0.02$, based on previously reported breast cancer incidence statistics~\citep{cr_prevalence_stats}. We also include an expertise constraint by allowing $\betadelta$ to deviate from 0 only for features which plausibly influence a person's probability of being tested beyond disease risk. We do not place the expertise constraint on (i) racial/socioeconomic features, due to disparities in healthcare access~\citep{chen2021ethical,pierson2020assessing,shanmugam2021quantifying}; (ii) genetic features, since genetic information may be unknown or underused~\citep{samphao2009diagnosis}; and (iii) age, due to age-based breast cancer testing policies~\citep{nhs_screening_doc}. In Appendix \ref{sec:appendix_robustness_experiments} Figures \ref{fig:noise}, \ref{fig:alpha}, and \ref{fig:prevalence}, we run robustness experiments.
\begin{figure}[]
\vspace{-2em}
  \begin{center}
    \includegraphics[width=0.65\textwidth]{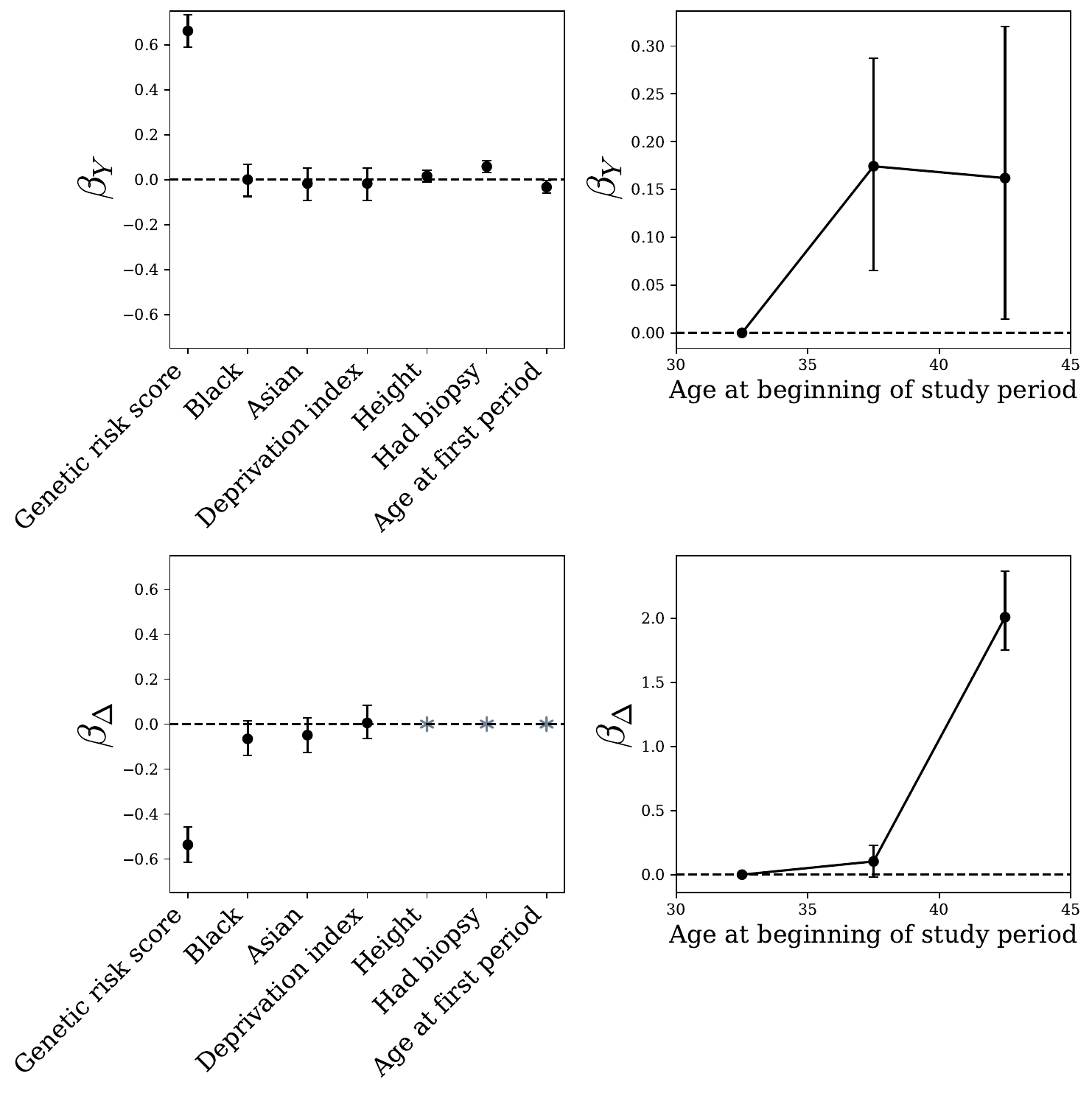}
  \end{center}
  \caption{Estimated $\betaY$ (top) capture known cancer risk factors: genetic risk, previous biopsy, age at first period (menarche), and age~\citep{nih_risk_tool,yanes2020clinical}. Estimated $\betadelta$ (bottom) capture the underuse of genetic information (left) and known age-based testing policies (right). Points indicate posterior means and vertical lines indicate 95\% confidence intervals. Gray asterisks indicate coefficients set to 0 by the expertise constraint.}
  \label{fig:coeff_forest}
  \vspace{-1.5em}
\end{figure}

In Figure \ref{fig:coeff_forest}, we plot the inferred coefficients for the fitted model. The model infers a large $\sigma^2=5.1$ (95\% CI, 3.7-6.8), highlighting the importance of unobservables. In Appendix \ref{sec:baselines} Figure \ref{fig:baselines}, we also compare our model’s performance to a suite of additional baselines, including (i) baselines trained solely on the tested population, (ii) baselines which treat the untested population as negative, and (iii) additional baselines commonly used in selective labels settings~\citep{rastogi2023learn}. Collectively, these baselines all suffer from various issues our model does not, including learning implausible age trends inconsistent with prior literature or worse predictive performance. 
\subsection{Validating the model} 
\label{sec:validating_model}
Validating models in real-world selective labels settings is difficult because outcomes are not observed for the untested. Still, we leverage the rich data in the UK Biobank to validate our model in four ways. 
\paragraph{Inferred risk predicts breast cancer diagnoses:} 
Verifying that inferred risk predicts diagnoses among the \emph{tested} population is straightforward. Since $Y$ is observed for the tested population, we check (on a test set) whether people with higher inferred risk ($p(Y_i=1|X_i)$) are more likely to be diagnosed with cancer ($Y_i = 1$). People in the highest inferred risk quintile\footnote{Reporting outcome rates by inferred risk quintile or decile is a common metric in health risk prediction settings~\citep{mullainathan2022diagnosing,einav2018predictive,obermeyer2019dissecting}.} have $3.3\times$ higher true risk of cancer than people in the lowest quintile (6.0\% vs 1.8\%). Verifying that inferred risk predicts diagnoses among the \emph{untested} population is less straightforward because $Y_i$ is not observed. We leverage that a subset have a \emph{follow-up} visit (i.e., an observation after the initial 10-year study period) to show that inferred risk predicts cancer diagnosis at the follow-up. For the subset of the untested population who attend a follow-up visit, people in the highest inferred risk quintile have $2.5\times$ higher true risk of cancer during the follow-up period than people in the lowest quintile (4.1\% vs 1.6\%).\footnote{AUC amongst the tested population is 0.63 and amongst the untested population that attended a followup is 0.63. These AUCs are similar to past predictions which use similar feature sets~\citep{yala2021toward}. For instance, the Tyrer-Cuzick~\citep{tyrer2004breast} and Gail~\citep{gail1989projecting} models achieved AUCs of 0.62 and 0.59.}
\paragraph{Inferred unobservables correlate with known unobservables:} For each person, our model infers a posterior over unobservables $p(Z_i | X_i, T_i, Y_i)$. We confirm that the inferred posterior mean of unobservables correlates with a true unobservable---whether the person has a family history of breast cancer. This is an unobservable because it influences both $T_i$ and $Y_i$ but is not included in the data given to the model.\footnote{Although UKBB has family history data, we do not include it as a feature both so we can use it as validation and because we do not have information on \emph{when} family members are diagnosed. So we cannot be sure that the measurement of family history precedes the measurement of $T_i$ and $Y_i$, as is desirable for features in $X_i$.} People in the highest inferred unobservables quintile are $2.1\times$ likelier to have a family history of cancer than people in the lowest quintile (15.6\% vs 7.5\%).
\paragraph{$\betaY$ captures known cancer risk factors:} $\betaY$ measures each feature's contribution to risk. The top left plot in Figure \ref{fig:coeff_forest} shows that the inferred $\betaY$ captures known cancer risk factors. Cancer risk is strongly correlated with genetic risk, and is also correlated with previous breast biopsy, age, and younger age at first period (menarche)~\citep{nih_risk_tool, yanes2020clinical}. 
\paragraph{$\betadelta$ captures known public health policies:} In the UK, all women aged 50-70 are invited for breast cancer testing every 3 years~\citep{nhs_screening_doc}. Our study period spans 10 years, so we expect women who are 40 or older at the start of the study period (50 or older at the end) to have an increased probability of testing when controlling for true cancer risk. The bottom right plot in Figure \ref{fig:coeff_forest} shows this is the case, since the $\betadelta$ indicator for ages 40-45 is greater than the indicators for ages $<$35 and 35-39.
\begin{wrapfigure}{r}{0.5\textwidth}
\vspace{1em}
  \begin{center}
    \includegraphics[width=0.48\textwidth]{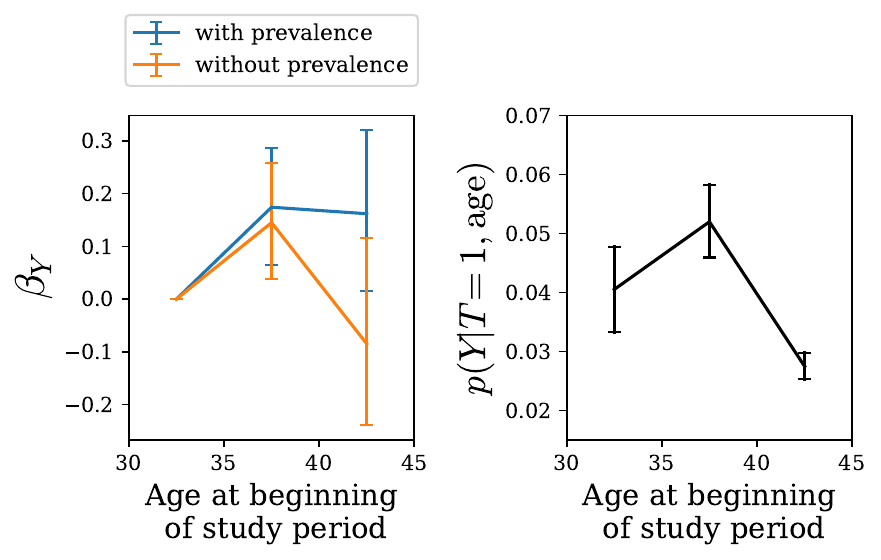}
  \end{center}
  \caption{Without the prevalence constraint, the model learns that cancer risk first increases and then decreases with age (left orange), contradicting prior literature~\citep{cancer_risk_predictors2,cr_prevalence_stats,us2013us,campisi2013aging}. This incorrect inference occurs because the tested population has the same misleading age trend (right). In contrast, the prevalence constraint encodes that the (younger) untested population has lower risk, allowing the model to learn a more accurate age trend (left blue).}
  \label{fig:age_coeff}
  \vspace{-4em}
\end{wrapfigure}
\subsection{Assessing historical\\testing decisions} 

\label{sec:human_decisions}
Non-zero components of $\betadelta$ indicate features that affect a person's probability of being tested even when controlling for their disease risk. The bottom left plot in Figure \ref{fig:coeff_forest} plots the inferred $\betadelta$, revealing that genetic information is underused. While genetic risk is strongly predictive of $Y_i$, its negative $\betadelta$ indicates that people at high genetic risk are tested less than expected given their risk. This is plausible, given that their genetic information may not have been available to guide decision-making. The model also infers negative point estimates for $\betadelta$ for Black and Asian women, consistent with known racial disparities in breast cancer testing~\citep{makurumidze2022addressing} as well as broader racial inequality in healthcare and other domains~\citep{nazroo2007black,zink2023race,movva2023coarse,obermeyer2019dissecting,franchi2023detecting,otu2020one,devonport2023systematic}. However, both confidence intervals overlap zero (due to the small size of these groups in our dataset).

\subsection{Comparison to model without prevalence constraint}
\label{sec:comparison_without_prevalence}
The prevalence constraint also guides the model to more plausible inferences. We compare the model fit with and without a prevalence constraint. As shown in the left plot in Figure \ref{fig:age_coeff}, without the prevalence constraint, the model learns that cancer risk first increases with age and then falls, contradicting prior epidemiological and physiological evidence~\citep{cancer_risk_predictors2, cr_prevalence_stats,us2013us,campisi2013aging}. This is because, due to the age-based testing policy in the UK~\citep{nhs_screening_doc}, being tested for breast cancer before age 50 is unusual. Thus, the tested population under age 50 is non-representative because their risk is much higher than the corresponding untested population. The prevalence constraint guides the model to more plausible  inferences by preventing the model from predicting that a large fraction of the untested (younger) population has the disease.

\section{Related work}
\label{sec:related_work}
Selective labels problems occur in many domains, including hiring, insurance, government inspections, tax auditing, recommender systems, lending, healthcare, education, welfare services, wildlife protection, and criminal justice~\citep{lakkaraju2017selective,jung2020simple,kleinberg2018human,bjorkegren2020behavior,jung2018omitted,jehi2020individualizing,mcdonald2021derivation,lauferend,mcwilliams2019towards,crook2004does,hong2018predicting,parker2019predicting,sun2011predicting,kansagara2011risk,waters2014grade,bogen2019all,jawaheer2010comparison,wu2017understanding,coston2020counterfactual,de2021leveraging,pierson2020assessing,pierson2020large,simoiu2017problem,mullainathan2022diagnosing,henderson2022integrating,gholami2019stay,farahani2020explanatory,underreporting22,cai2020fair, daysal2022economic, guerdan2023counterfactual, chan2022selection, jiang2021learning,chien2023algorithmic, jia2019anthropogenic}. As such, there are related literatures in machine learning and causal inference~\citep{coston2020counterfactual,schulam2017reliable,lakkaraju2017selective,kleinberg2018human,shimodaira2000improving,de2021leveraging,levine2020offline,koh2021wilds,sagawa2021extending,kaur2022modeling,sahoo2022learning, cortesgomez2023statistical}, econometrics~\citep{mullainathan2022diagnosing,rambachan2022counterfactual,heckman1976common,hull2021marginal,kunzel2019metalearners,shalit2017estimating,wager2018estimation,alaa2018limits}, statistics and Bayesian models~\citep{ilyas2020theoretical,daskalakis2021efficient,mishler2021fade,jung2020bayesian}, and epidemiology~\citep{groenwold2012dealing,perkins2018principled}. We extend this literature by providing constraints which both theoretically and empirically improve parameter inference. We now describe the three lines of work most closely related to our modeling approach.
\paragraph{Generalized linear mixed models (GLMMs):} Our model is closely related to GLMMs~\citep{gelman2013bayesian,stroup2012generalized,lum2022closer}, which model observations as a function of both observed features $X_i$ and unobserved ``random effects'' $Z_i$. We extend this literature by (i) proposing and analyzing a novel model to capture our selective labels setting; (ii) incorporating the uniform distribution of unobservables, as opposed to the normal distribution typically used in GLMMs, to yield more tractable inference; and most importantly (iii) incorporating healthcare domain constraints into GLMMs to improve model estimation.
\paragraph{Improving robustness to distribution shift using domain information:} The selective labels setting represents a specific type of distribution shift from the tested to untested population. Previous work shows that generic methods often fail to perform well across all types of distribution shifts~\citep{gulrajani2020search,koh2021wilds,sagawa2021extending,wiles2021fine,kaur2022modeling} and that incorporating domain information can improve performance.~\citet{gao2023out} proposes \emph{targeted augmentations}, which augment the data by randomizing known spurious features while preserving robust ones.~\citet{tellez2019quantifying} presents an example of this strategy for histopathology slide analysis.~\citet{kaur2022modeling} shows that modeling the data generating process is necessary for generalizing across distribution shifts. Motivated by this, we propose a data generating process suitable for selective labels settings and show that using domain information improves performance.
\paragraph{Breast cancer risk estimation:} There are many related works on estimating breast cancer risk~\citep{daysal2022economic, yala2019deep, yala2021toward, yala2022optimizing,shen2021artificial}. Our work complements this literature by proposing a Bayesian model which captures the selective labels setting and incorporating domain constraints to improve model estimation. While a linear model suffices for the low-dimensional features used in our case study, our approach naturally extends to more complex inputs (e.g., medical images) and deep learning models sometimes used in breast cancer risk prediction~\citep{yala2019deep,yala2021toward, yala2022optimizing}.
\section{Discussion}
We propose a Bayesian model class to infer risk and assess historical human decision-making in selective labels settings, which commonly occur in healthcare and other domains. We propose the prevalence and expertise constraints which we show both theoretically and empirically improve parameter inference. We apply our model to cancer risk prediction, validate its inferences, show it can identify suboptimalities in test allocation, and show the prevalence constraint prevents misleading inferences.

A natural future direction is applying our model to other healthcare settings, where a frequent practice is to train risk-prediction models only on the tested population~\citep{jehi2020individualizing,mcdonald2021derivation,farahani2020explanatory}. This is far from optimal both because only a small fraction of the population is tested, increasing variance, and because the tested population is highly non-representative, increasing bias. The paradigm we propose offers a solution to both problems. Using data from the entire population reduces variance, and modeling the distribution shift and constraining inferences on the untested population reduces bias. Beyond healthcare, other selective labels domains may have other natural domain constraints: for example, randomly assigned human decision-makers~\citep{kleinberg2018human} or repeated measurements of the same individual~\citep{lum2022closer}. Beyond selective labels, our model represents a concrete example of how domain constraints can improve inference in the presence of distribution shift.
\subsubsection*{Acknowledgments}
The authors thank Gabriel Agostini, Sivaramakrishnan Balachandar, Serina Chang, Erica Chiang, Avi Feller, Eran Halperin, Andrew Ilyas, Pang Wei Koh, Ben Laufer, Zhi Liu, Smitha Milli, Sendhil Mullainathan, Josue Nassar, Kenny Peng, Ashesh Rambachan, Richa Rastogi, Evan Rose, Shuvom Sadhuka, Jacob Steinhardt, Robert Tillman, and Manolis Zampetakis for helpful conversations. This research was supported by a Google Research Scholar award, NSF CAREER \#2142419, a CIFAR Azrieli Global scholarship, Optum, a LinkedIn Research Award, the Abby Joseph Cohen Faculty Fund, and NSF GRFP Grant DGE \#2139899. This research has been conducted using the UK Biobank Resource under Application Number 72589. Any opinions, findings, conclusions, or recommendations expressed in this material are those of the authors and do not necessarily reflect the views of the funders.
\bibliography{iclr2024_conference}
\bibliographystyle{iclr2024_conference}

\newpage
\setcounter{figure}{0}
\makeatletter 
\renewcommand{\thefigure}{S\@arabic\c@figure}
\makeatother

\setcounter{table}{0}
\makeatletter 
\renewcommand{\thetable}{S\@arabic\c@table}
\makeatother
\appendix
\section{Calculating disease prevalence}
\label{sec:appendix_prevalence_constraint}
To implement the prevalence constraint, we assume that the \textit{disease prevalence}, or average value of $Y$ across the population, is at least approximately known. This assumption is plausible in medical settings because estimating prevalence is the focus of substantial public health research. Methods to calculate prevalence include serology, where blood samples are used to detect specific antibodies or antigens of a disease~\citep{joseph1995bayesian}; stool or wastewater testing for disease markers~\citep{joseph1995bayesian, mcmahan2021covid}; genetic methods, where genomic registries can be analyzed to calculate allele frequency and estimate disease prevalence~\citep{schrodi2015prevalence}; autopsy reports for a particular disease~\citep{bell2015prevalence}; and administrative data collected by primary, outpatient, and inpatient care centers~\citep{wirehn2007estimating}. Additionally, our Bayesian formulation can incorporate approximate prevalence estimates (e.g. bounded estimates), and these bounds can be estimated using the sensitivity and specificity of the prevalence estimation method~\citep{manski2021estimating, manski2020bounding, mullahy2021embracing}.
\section{Proofs}
\label{sec:appendix_proofs}
\paragraph{Proof outline:} In this section, we provide three proofs to show why domain constraints improve parameter inference. We start by showing that the well-studied Heckman correction model~\citep{heckman1976common,heckman1979sample} is a special case of the general model in \eqref{eq:DGP} (Proposition \ref{heckman}). It is known that placing constraints on the Heckman model can improve parameter inference~\citep{lewbel2019identification}. We show that our proposed prevalence and expertise constraints have a similar effect by proving that our proposed constraints never worsen the precision of parameter inference (Proposition \ref{identifiability}). We then provide conditions under which our constraints strictly improve precision (Proposition \ref{strict_inequality}).
\paragraph{Notation and assumptions:} Below, we use $\Phi$ to denote the normal CDF, $\phi$ the normal PDF, and $\betaT = \alpha\betaY+\betadelta$. Let $\capX$ be the matrix of observable features. We assume that the first column of $\capX$ corresponds to the intercept; $\capX$ is zero mean for all columns except the intercept; and the standard identifiability condition that our data matrix is full rank, i.e., $\capX^T\capX$ is invertible. We also assume that $\alpha>0$.

We start by defining the Heckman correction model.
\heckmandef*
In other words, $T_i=1$ if a linear function of $\capXi$ plus some unit normal noise $u_i$ exceeds zero. $Y_i$ is a linear function of $\capXi$ plus normal noise $Z_i$ with variance $\tilde{\sigma}^2$. Importantly, the noise terms $Z_i$ and $u_i$ are \emph{correlated}, with covariance $\tilde{\rho}$. The model parameters are $\tilde{\theta} \triangleq (\tilde{\rho}, \tilde{\sigma}^2, \tildebetaT, \tildebetaY)$. We use tildes over the Heckman model parameters to distinguish them from the parameters in our original model in \eqref{eq:DGP}. We now prove Proposition \ref{heckman}.
\heckman*
\begin{proof}
If we substitute in the value of $r_i$, the equation for $Y_i$ is equivalent to that in the Heckman model. So it remains only to show that $T_i$ in \eqref{eq:heckmanmodel} can be rewritten in the form in \eqref{eq:properheckmanmodel}. We first rewrite \eqref{eq:heckmanmodel} in slightly more convenient form:
\begin{align*}
T_i&\sim \text{Bernoulli}(\Phi(\alpha r_i+\capXi^T\betadelta)) \rightarrow\\ 
T_i&\sim \text{Bernoulli}(\Phi(\alpha (\capXi^T\betaY + Z_i)+\capXi^T\betadelta)) \rightarrow\\ 
T_i&\sim \text{Bernoulli}(\Phi(\capXi^T(\alpha \betaY + \betadelta) + \alpha Z_i)) \rightarrow\\ 
T_i&\sim \text{Bernoulli}(\Phi(\capXi^T \betaT + \alpha Z_i))\, .
\end{align*}
 We then apply the latent variable formulation of the probit link:
\begin{align*}
T_i&\sim \text{Bernoulli}(\Phi(\capXi^T \betaT + \alpha Z_i)) \rightarrow \\ 
T_i &= \mathbbm{1}[\capXi^T\betaT + \alpha Z_i + \epsilon_i > 0], \epsilon_i \sim \mathcal{N}(0, 1)\, ,
\end{align*}
where $\alpha Z_i + \epsilon_i$ is a normal random variable with standard deviation $\sqrt{\alpha^2\sigma^2 + 1}$. We divide through by this factor to rewrite the equation for $T_i$:
\begin{align*}
T_i &= \mathbbm{1}[\capXi^T\tildebetaT + u_i > 0]\, ,
\end{align*}
which is equivalent to \eqref{eq:properheckmanmodel}. Here, $\tildebetaT =  \frac{\beta_T}{\sqrt{\alpha^2\sigma^2 + 1}}$ and $u_i=\frac{\alpha Z_i + \epsilon_i}{\sqrt{\alpha^2\sigma^2 + 1}}$ is a unit-scale normal random variable whose covariance with $Z_i$ is
\begin{align*}
\text{cov}\left(\frac{\alpha Z_i + \epsilon_i}{\sqrt{\alpha^2\sigma^2 + 1}}, Z_i\right) &= \bbE\left(\frac{\alpha Z_i + \epsilon_i}{\sqrt{\alpha^2\sigma^2 + 1}}\cdot Z_i\right)-\bbE\left(\frac{\alpha Z_i + \epsilon_i}{\sqrt{\alpha^2\sigma^2 + 1}}\right)\bbE\left(Z_i\right)\\
&=\frac{\alpha\bbE\left(Z_i^2\right)}{\sqrt{\alpha^2\sigma^2 + 1}}\\
&=\frac{\alpha \sigma^2}{\sqrt{\alpha^2\sigma^2 + 1}}\, .
\end{align*}
Thus, the special case of our model in \eqref{eq:heckmanmodel} is equivalent to the Heckman model, where the mapping between the parameters is:
\begin{align}
\begin{split}
\tildebetaY &= \betaY\\
\tilde{\sigma}^2 &= \sigma^2\\
\tildebetaT &=  \frac{\betaT}{\sqrt{\alpha^2\sigma^2 + 1}} \\
\tilde{\rho} &= \frac{\alpha \sigma^2}{\sqrt{\alpha^2\sigma^2 + 1}}\, .
\end{split}
 \label{eq:mapping}
\end{align}
\end{proof}
As described in \cite{lewbel2019identification}, the  Heckman correction model is identified without any further assumptions. It then follows that the special case of our model in \eqref{eq:heckmanmodel} is identified without further constraints. One can simply estimate the Heckman model, which by the mapping in \eqref{eq:mapping} immediately yields estimates of $\betaY$ and $\sigma^2$. Then, the equation for $\tilde{\rho}$ can be solved for $\alpha$, yielding a unique value since $\alpha>0$. Similarly the equation for $\tildebetaT$ yields the estimate for $\betaT$ (and thus $\betadelta$).  

While the Heckman model is identified without further constraints, this identification is known to be very weak, relying on functional form assumptions~\citep{lewbel2019identification}. To mitigate this problem, when the Heckman model is used in the econometrics literature it is typically estimated with constraints on the parameters. In particular, a frequently used constraint is an \emph{exclusion restriction}: there must be at least one feature with a non-zero coefficient in the equation for $T$ but not $Y$. While this constraint differs from the ones we propose, one might expect our proposed prevalence and expertise constraints to have a similar effect and improve the precision of parameter inference. We make this precise through Proposition \ref{identifiability}. 

Throughout the results below, we analyze the posterior distribution of model parameters given the observed data: $g(\theta) \triangleq p(\theta|\capX, T, Y)$. We show that constraining the value of any one parameter (through the prevalence or expertise constraint) will not worsen the posterior variance of the other parameters. In particular, constraining a parameter $\thetacon$ to a value drawn from its posterior distribution will not in expectation increase the posterior variance of any other unconstrained parameters $\thetaunc$. To formalize this, we define the \emph{expected conditional variance}:
\conditionalvariance*
\identifiability*
\begin{proof}
The proof follows from applying the law of total variance to the posterior distribution $g$. The law of total variance states that:
\begin{align*}
     \textrm{Var}(\thetaunc) = \mathbb{E}[\textrm{Var}(\thetaunc|\thetacon)] +  \textrm{Var}(\mathbb{E}[\thetaunc|\thetacon])\, .
\end{align*}
Since $\textrm{Var}(\mathbb{E}[\thetaunc|\thetacon])$ is non-negative,
\begin{align*}
     \mathbb{E}[\textrm{Var}(\thetaunc|\thetacon)]\leq\textrm{Var}(\thetaunc)\, .
\end{align*}
Additionally, if $\mathbb{E}[\thetaunc|\thetacon]$ is non-constant in $\thetacon$ then $\textrm{Var}(\mathbb{E}[\thetaunc|\thetacon])$ is strictly positive. Thus the strict inequality follows.
\end{proof}
We now discuss how Proposition \ref{identifiability} applies to our proposed constraints and the Heckman model. Both the prevalence and expertise constraints fix the value of at least one parameter. The prevalence constraint fixes the value of ${\betaY}_0$ and the expertise constraint fixes the value of ${\betadelta}_d$ for some $d$. Thus by Proposition \ref{identifiability}, we know that the prevalence and expertise constraints will not increase the variance of any model parameters, and will strictly reduce them as long as the posterior expectations of the unconstrained parameters are non-constant in the constrained parameters.

We now show that when $\tildebetaT$ is known, the prevalence constraint strictly reduces variance. The setting where $\tildebetaT$ is known is a natural one because $\tildebetaT$ can be immediately estimated from the observed data $\capX$ and $T$, and previous work in both econometrics and statistics thus have also considered this setting~\citep{heckman1976common, ilyas2020theoretical}. With additional assumptions, we also show that the expertise constraint strictly reduces variance. We derive these results in the setting with flat priors for algebraic simplicity. However, analogous results also hold under other natural choices of prior (e.g., standard conjugate priors for Bayesian linear regression~\citep{jackman2009bayesian}). In the results below, we analyze the conditional mean of $Y$ conditioned on $T=1$. Thus, we start by defining this value.
\begin{restatable}[Conditional mean of $Y$ conditioned on $T=1$]{lemma}{conditional_mean}
\label{conditional_mean}
Past work has shown that the expected value of $Y_i$ when $T_i = 1$ is~\citep{heckman_model}:
\begin{align*}
\bbE[Y_i|T_i=1] &= \bbE[Y_i|\capX_i^T\tildebetaT + u > 0]\\ 
&= \capX_i\tildebetaY + \tilde{\rho}\tilde{\sigma}\frac{\phi(\capX_i\tildebetaT)}{\Phi(\capX_i\tildebetaT)}\, ,
\end{align*}
where $\Phi$ denotes the normal CDF, $\phi$ the normal PDF, and $\frac{\phi(\capX\tildebetaT)}{\Phi(\capX\tildebetaT)}$ the inverse Mills ratio. This can be more succinctly represented in matrix notation as
\begin{align*}
\bbE[Y_i|T_i=1] = M \theta\, ,
\end{align*}
where $M = [\capX_{T=1}; \frac{\phi(\capX_{T=1}\tildebetaT)}{\Phi(\capX_{T=1}\tildebetaT)}]\in\mathbb{R}^{N_{T=1}\times (d+1)}$, $\theta = [{\tildebetaY}, \tilde{\rho}\tilde{\sigma}]\in\mathbb{R}^{d+1}$, $\capX_{T=1}$ denotes the rows of $X$ corresponding to $T=1$, and $N_{T=1}$ is the number of rows of $\capX$ for which $T=1$.
\end{restatable}
\begin{restatable}[]{proposition}{strict_inequality}
\label{strict_inequality}
Assume $\tildebetaT$ is fixed and flat priors on all parameters. Additionally, assume the standard identifiability condition that the matrix $M = [\capX_{T=1}; \frac{\phi(\capX_{T=1}\tildebetaT)}{\Phi(\capX_{T=1}\tildebetaT)}]$ is full rank. Then, in expectation, constraining a component of $\tildebetaY$ in the Heckman correction model strictly reduces the posterior variance of the other model parameters. The prevalence constraint does this without any further assumptions, and the expertise constraint does this if $\tilde{\rho}$ and $\tilde{\sigma}^2$ are fixed.
\end{restatable}
\begin{proof}
We will start by showing that when $\tildebetaT$ is fixed, constraining a component of $\tildebetaY$ strictly reduces the variance of the other model parameters. From the definition of the conditional mean of $Y$ conditioned on $T=1$ (Lemma \ref{conditional_mean}), we get
\begin{align*}
\bbE[Y_i|T_i=1] = M \theta\, .
\end{align*}
Under flat priors on all parameters, the posterior expectation of the model parameters given the observed data $\{\capX, T, Y\}$ is simply the standard ordinary least squares solution given by the normal equation~\citep{jackman2009bayesian}: 
\begin{align*}
\mathbb{E}[\theta|\capX, T, Y]=(M^TM)^{-1}M^TY\, .
\end{align*}
By assumption, $M$ is full rank, so $M^TM$ is invertible.

When ${\tilde{\boldsymbol{\beta}}_{\boldsymbol{Y}_d}}$ is constrained to equal to ${\tilde{\boldsymbol{\beta}}_{\boldsymbol{Y}_d}}^*$ for some component $d$, the equation instead becomes:
\begin{align*}
\mathbb{E}[\theta_{-d}|{\tilde{\boldsymbol{\beta}}_{\boldsymbol{Y}_d}}={\tilde{\boldsymbol{\beta}}_{\boldsymbol{Y}_d}}^*, \capX, T, Y]=({M^T_{-d}}M_{-d})^{-1}{M^T_{-d}}(Y - X_{{T=1}_d}{\tilde{\boldsymbol{\beta}}_{\boldsymbol{Y}_d}}^*)\, .
\end{align*}

We use the subscript $-d$ notation to indicate that we no longer estimate the component $d$. Here, $M_{-d} = [\capX_{{T=1}_{-d}}; \frac{\phi(\capX_{T=1}\tildebetaT)}{\Phi(\capX_{T=1}\tildebetaT)}]\in\mathbb{R}^{N_{T=1}\times d}$ and $\theta_{-d} = [{\tilde{\boldsymbol{\beta}}_{\boldsymbol{Y}_{-d}}}, \tilde{\rho}\tilde{\sigma}]\in\mathbb{R}^{d}$. Since $X_{{T=1}_d}$ is nonzero and $M$ is full rank, it follows that $\mathbb{E}[\theta_{-d}|{\tilde{\boldsymbol{\beta}}_{\boldsymbol{Y}_d}}={\tilde{\boldsymbol{\beta}}_{\boldsymbol{Y}_d}}^*, \capX, T, Y]$ is not constant in ${\tilde{\boldsymbol{\beta}}_{\boldsymbol{Y}_d}}^*$. Thus by Proposition \ref{identifiability}, constraining ${\tilde{\boldsymbol{\beta}}_{\boldsymbol{Y}_d}}$ reduces the variance of the parameters in $\theta_{-d}$ (${\tilde{\boldsymbol{\beta}}_{\boldsymbol{Y}_d'}}$ for $d'\neq d$ and $\tilde{\rho}\tilde{\sigma}$).

We will now show that both the prevalence and expertise constraints constrain a component of $\tildebetaY$. Assuming the standard condition that columns of $X$ are zero-mean except for an intercept column of ones, the prevalance constraint fixes 
\begin{align*}
    \mathbb{E}_Y[Y] &=  \bbE_Y[\bbE_X[\bbE_Z[Y | X,Z]]] \\
    &= \bbE_X[\bbE_Z[\capX^T\betaY + Z]] \\
    &= {\betaY}_0\, ,
\end{align*}
where ${\betaY}_0$ is the 0th index (intercept term) of $\betaY$. The expertise constraint also fixes a component of $\tildebetaY$ if $\tilde{\rho}$ and $\tilde{\sigma}^2$ are fixed. This can be shown by algebraically rearranging \eqref{eq:mapping} to yield
\begin{align*}
\tildebetaY &= \tildebetaT\frac{\tilde{\sigma}^2}{\tilde{\rho}}-\betadelta\frac{\tilde{\sigma}\sqrt{\tilde{\sigma}^2-\tilde{\rho}^2}}{\tilde{\rho}}\, .
\end{align*}
\end{proof}
While we derive our theoretical results for the Heckman correction model, in both our synthetic experiments (\S \ref{sec:synthetic_experiments}) and our real-world case study (\S \ref{sec:real_data_experiments}) we validate that our constraints improve parameter inference beyond the special Heckman case.
\begin{figure}
\vspace{-2em}
  \begin{center}
    \includegraphics[width=0.9\textwidth]{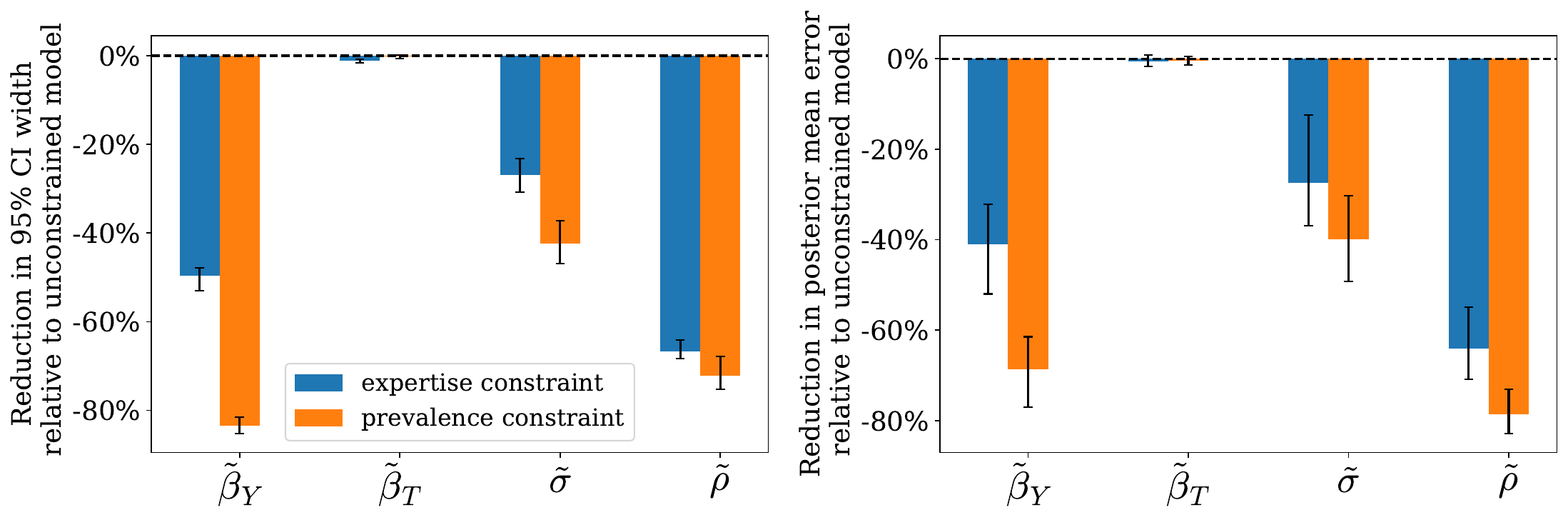}
  \end{center}
  \caption{Results using synthetic data from the Heckman model. The prevalence and expertise constraints each produce more precise and accurate inferences on this synthetic data. We plot the median across 200 synthetic datasets. Errorbars denote the bootstrapped 95\% confidence interval on the median.}
  \label{fig:heckman_plot}
  \vspace{-1.5em}
\end{figure}

\section{Derivation of the closed-form uniform unobservables model} 
\label{sec:appendix_uniform_model}
Conducting sampling for our general model described by \eqref{eq:DGP} is faster if the distribution of unobservables $f$ and link functions $h_Y$ and $h_T$ allow one to marginalize out $Z_i$ through closed-form integrals, since otherwise $Z_i$ must be sampled for each datapoint $i$, producing a high-dimensional latent variable which slows computation and convergence. Many distributions do not produce closed-form integrals when combined with a sigmoid or probit link function, which are two of the most commonly used links with binary variables.\footnote{Specifically, we search over the distributions in \cite{mclaughlin2001compendium}, combined with logit or probit links, and find that most combinations do not yield closed forms.} However, we \emph{can} derive closed forms for the special \emph{uniform unobservables} case described by \eqref{eq:uniform_model}.

\begin{figure}
\vspace{-2em}
  \begin{center}
    \includegraphics[width=0.9\textwidth]{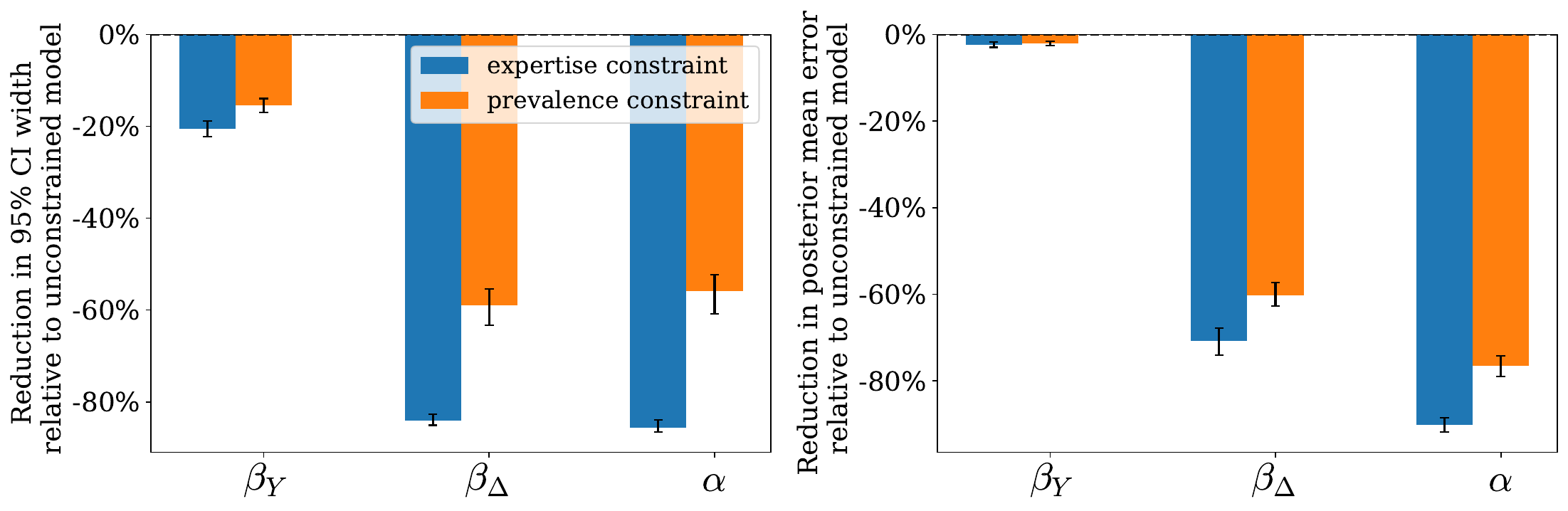}
  \end{center}
  \caption{Results using synthetic data from the Bernoulli-sigmoid model with normal unobservables and fixed $\sigma^2$. The prevalence and expertise constraints each produce more precise and accurate inferences on this synthetic data. We plot the median across 200 synthetic datasets. Errorbars denote the bootstrapped 95\% confidence interval on the median.}
  \label{fig:normal_fixed_sigma}
  \vspace{-1.5em}
\end{figure}

Below, we leave the $i$ subscript implicit to keep the notation concise. When computing the log likelihood of the data, to marginalize out $Z$, we must be able to derive closed forms for the following three integrals:
\begin{align*}
\begin{split}
p(Y=1, T=1 | X) &= \int_Z p(Y=1, T=1|X, Z)f(Z)dZ \\
p(Y=0, T=1 | X) &= \int_Z p(Y=0, T=1|X, Z)f(Z)dZ \\
p(T=0 | X) &= \int_Z p(T=0|X, Z)f(Z)dZ \, ,
\end{split}
\end{align*}
since the three possibilities for an individual datapoint are $\{Y=1, T=1\}$, $\{Y=0, T=1\}$, $\{T=0\}$. To implement the prevalence constraint (which fixes the $\bbE[Y]$), we also need a closed form for the following integral:
\begin{align*}
\begin{split}
p(Y=1 | X) &= \int_Z p(Y=1|X, Z)f(Z)dZ \, .
\end{split}
\end{align*}
For the uniform unobservables model with $\alpha = 1$, the four integrals have the following closed forms, where below we define $A=e^{\capX^T\betaT}$ and $B=e^{\capX^T\betaY}$:
\begin{align*}
\begin{split}
p(Y=1, T=1 | X) &= \frac{1}{\sigma \left(A-B\right)}
\bigg(\sigma \left(A-B\right) - A \log{\left(\left(B + 1\right) A^{-1} \right)}\\
&+ A \log{\left(\left(Be^{\sigma} + 1\right) A^{-1}e^{- \sigma} \right)}+  B \log{\left(\left(A + 1\right) A^{-1} \right)}\\
&- B \log{\left(\left(Ae^{\sigma} + 1\right) A^{-1}e^{- \sigma} \right)}\bigg) \\ 
p(Y=0, T=1|X) &= \frac{1}{\sigma \left(A - B\right)}\bigg(
\big(-\log{\left(\left(A + 1\right) A^{-1} \right)} + \log{\left(\left(B + 1\right) A^{-1} \right)} \\
&+ \log{\left(\left(Ae^{\sigma} + 1\right) A^{-1}e^{- \sigma} \right)} - \log{\left(\left(Be^{\sigma} + 1\right) A^{-1}e^{- \sigma} \right)}\big) A\bigg) \\ 
p(T=0 | X) &= \frac{\log{\left(1 + A^{-1} \right)} - \log{\left(A^{-1}e^{- \sigma} + 1 \right)}}{\sigma} \\
p(Y=1 | X) &= \frac{\sigma - \log{\left(1 + B^{-1}\right)} + \log{\left(B^{-1}e^{- \sigma} + 1 \right)}}{\sigma} \, .
\end{split}
\end{align*}
The integrals also have closed forms for other integer values of $\alpha$ (e.g., $\alpha=2$) allowing one to perform robustness checks with alternate model specifications (see Appendix \ref{sec:appendix_robustness_experiments} Figure \ref{fig:alpha}).

\section{Synthetic experiments}
\label{sec:appendix_syn_experiments}
We first validate that the prevalence and expertise constraints improve the precision and accuracy of parameter inference for the Heckman model described in \eqref{eq:properheckmanmodel}. We then extend beyond this special case and examine various Bernoulli-sigmoid instantiations of our general model in \eqref{eq:DGP}, which assume a binary outcome variable $Y$. With a binary outcome, models are known to be more challenging to fit: for example, one cannot simultaneously estimate both $\alpha$ and $\sigma^2$ (so we must fix either $\alpha$ or $\sigma^2$), and models fit without constraints may fail to recover the correct parameters~\citep{stata, vandeven1981demand,toomet2008sample}. We assess whether our proposed constraints improve model estimation even in this more challenging case. Specifically, we extend beyond the Heckman model to the following data generating settings: (i) uniform unobservables and fixed $\alpha$, (ii) normal unobservables and fixed $\sigma^2$; (iii) normal unobservables and fixed $\alpha$; and (iv) other more complex models. For the uniform model, we conduct experiments only with fixed $\alpha$ (not fixed $\sigma^2$) because, as discussed above, this allows us to marginalize out $Z$. 

In all models, to incorporate the prevalence constraint into the model, we add a quadratic penalty to the model penalizing it for inferences that produce an inferred $\mathbb{E}[Y]$ that deviates from the true $\mathbb{E}[Y]$. To incorporate the expertise constraint into the model, we set the model parameters $\boldsymbol{\beta}_{\Delta_d}$ to be equal to 0 for all dimensions $d$ to which the expertise constraint applies.
\subsection{Heckman model}
\label{sec:appendix_heckman_model}
We first conduct synthetic experiments using the Heckman model defined in \eqref{eq:properheckmanmodel}. This model is identifiable without any further constraints, thus we estimate parameters $\theta \triangleq (\tilde{\rho}, \tilde{\sigma}^2, \tildebetaT, \tildebetaY)$. 

In the simulation, we use 5000 datapoints; 5 features (including the intercept column of 1s); $\capX$, $\boldsymbol{\beta}_Y$, and $\boldsymbol{\beta}_T$ drawn from unit normal distributions; and $\sigma \sim \mathcal{N}(2, 0.1)$. We draw the intercept terms $\boldsymbol{\beta}_{Y_0}\sim \mathcal{N}(-2,0.1)$ and $\boldsymbol{\beta}_{T_0}\sim \mathcal{N}(2,0.1)$. We assume the expertise constraint applies to $\boldsymbol{\beta}_{\Delta_2}=\boldsymbol{\beta}_{\Delta_3}=\boldsymbol{\beta}_{\Delta_4}=0$. Thus, by rearranging \eqref{eq:mapping}, we fix $\tildebetaY = \tildebetaT\frac{\tilde{\sigma}^2}{\tilde{\rho}}$. When calculating the results for $\tildebetaT$ and $\tildebetaY$, we do not include the dimensions along which we assume expertise since these dimensions are assumed to be fixed for the model with the expertise constraint. 

We show results in Figure \ref{fig:heckman_plot}. Both constraints generally produce more precise and accurate inferences for all parameters relative to the unconstrained model. The only exception is $\tildebetaT$, for which both models produce equivalently accurate and precise inferences. This is consistent with our theoretical results, which do not imply that the precision of inference for $\tildebetaT$ should improve. 
\begin{figure}
\vspace{-2em}
  \begin{center}
    \includegraphics[width=0.9\textwidth]{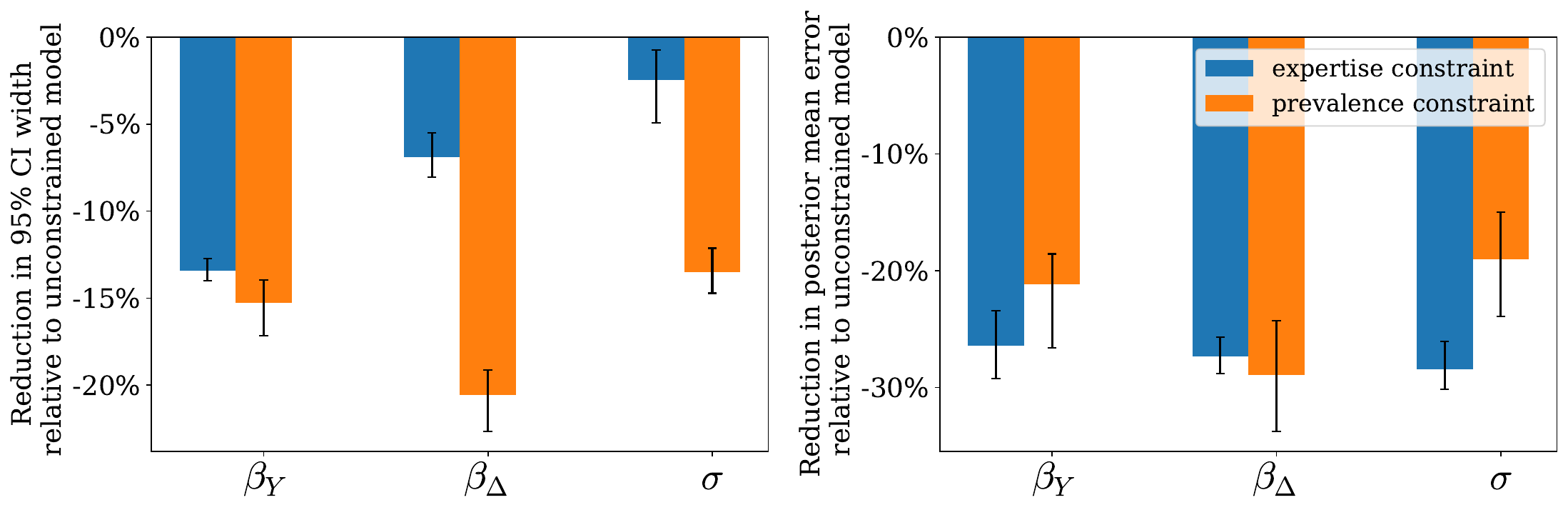}
  \end{center}
  \caption{Results using synthetic data from the Bernoulli-sigmoid model with normal unobservables and fixed $\alpha$. The prevalence and expertise constraints each produce more precise and accurate inferences on this synthetic data. We plot the median across 200 synthetic datasets. Errorbars denote the bootstrapped 95\% confidence interval on the median.}
  \label{fig:normal_fixed_alpha}
  \vspace{-1.5em}
\end{figure}
\label{sec:appendix_synthetic_experiments}
\subsection{Uniform unobservables model}
\label{sec:appendix_uniform_noise_model}
We now discuss our synthetic experiments using the Bernoulli-sigmoid model with uniform unobservables and $\alpha=1$ in \eqref{eq:uniform_model}. Our simulation parameters are similar to the Heckman model experiments. We use 5000 datapoints; 5 features (including the intercept column of 1s); $\capX$, $\boldsymbol{\beta}_Y$, and $\boldsymbol{\beta}_\Delta$ drawn from unit normal distributions; and $\sigma \sim \mathcal{N}(2, 0.1)$. We draw the intercept terms $\boldsymbol{\beta}_{Y_0}\sim \mathcal{N}(-2,0.1)$ and $\boldsymbol{\beta}_{\Delta_0}\sim \mathcal{N}(2,0.1)$ to approximately match $p(Y)$ and $p(T)$ in realistic medical settings, where disease prevalence is relatively low, but a large fraction of the population is tested because false negatives are more costly than false positives. We assume the expertise constraint applies to $\boldsymbol{\beta}_{\Delta_2}=\boldsymbol{\beta}_{\Delta_3}=\boldsymbol{\beta}_{\Delta_4}=0$. We show results in Figure \ref{fig:bar_plot}. When calculating the results for $\betadelta$, we do not include the dimensions along which we assume expertise since these dimensions are assumed to be fixed for the model with the expertise constraint. 
\begin{figure}
\vspace{-2em}
  \begin{center}
    \includegraphics[width=0.9\textwidth]{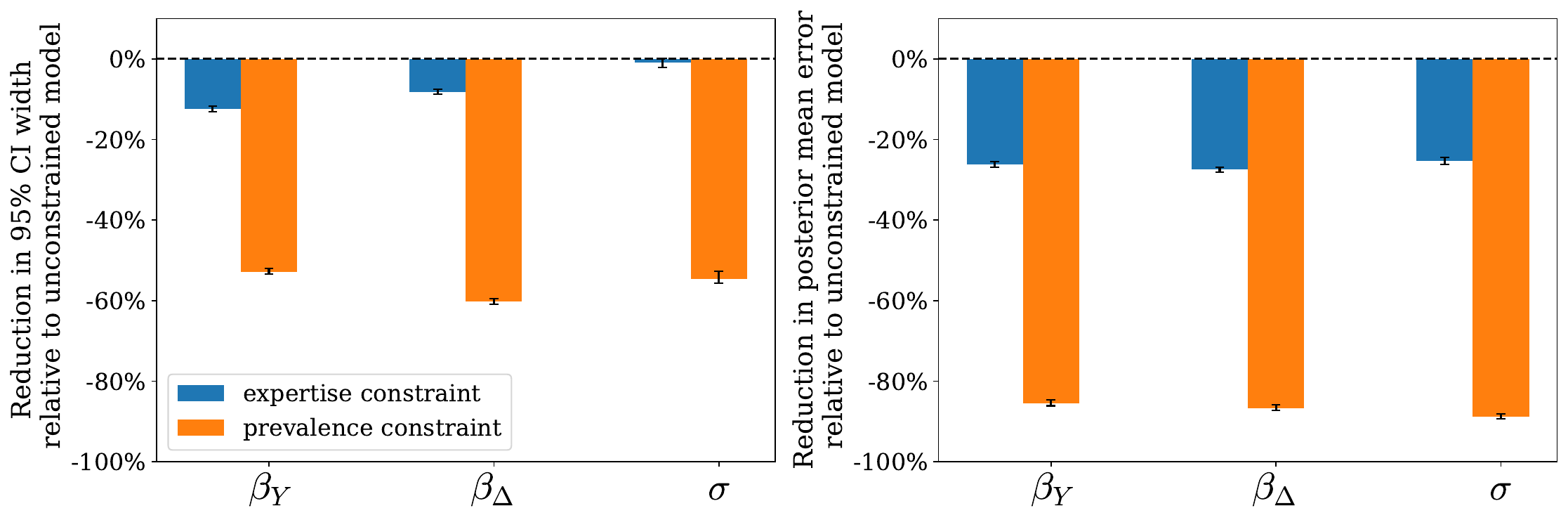}
  \end{center}
  \caption{The prevalence and expertise constraints still improve parameter inference when quadrupling the number of features relative to Figure \ref{fig:bar_plot}. Results are shown using synthetic data from the Bernoulli-sigmoid model with uniform unobservables. Both constraints produce more precise and accurate inferences on this synthetic data. We plot the median across 200 synthetic datasets. Errorbars denote the bootstrapped 95\% confidence interval on the median.}
  \label{fig:quadruple_features}
  \vspace{-1.5em}
\end{figure}
\subsection{Normal unobservables model}
\label{sec:appendix_normal_noise_model}
We also conduct synthetic experiments using the following Bernoulli-sigmoid model with normal unobservables: 
\begin{align}
\label{eq:normal_model}
\begin{split}
Z_i &\sim \mathcal{N}(0, \sigma^2)\\
r_i &= \capXi^T\betaY+Z_i\\
Y_i&\sim \text{Bernoulli}(\text{sigmoid}(r_i))\\
T_i&\sim \text{Bernoulli}(\text{sigmoid}(\alpha r_i+\capXi^T\betadelta))\, .
\end{split}
\end{align}
We show results for two cases: when $\sigma^2$ is fixed and when $\alpha$ is fixed. Because this distribution of unobservables does not allow us to marginalize out $Z$, it converges more slowly than the uniform unobservables model and we must use a smaller sample size for computational tractability.
\paragraph{Fixed $\sigma^2$:}
We use the same simulation parameters as the uniform model. We fix $\sigma^2=2$ and we draw $\alpha \sim N(1, 0.1)$. We show results in Figure \ref{fig:normal_fixed_sigma}. Both the prevalence and expertise constraints produce more precise and accurate inferences for all parameters relative to the unconstrained model.
\paragraph{Fixed $\alpha$:} We use the same simulation parameters as the uniform model, except we reduce the number of datapoints to 200. We fix $\alpha=1$ and we draw $\sigma^2 \sim N(2, 0.1)$. We show results in Figure \ref{fig:normal_fixed_alpha}. Both the prevalence and expertise constraints produce more precise and accurate inferences for all parameters relative to the unconstrained model. 

\subsection{More complex models}
To show our constraints are useful with more complex models, we ran two additional synthetic experiments on the Bernoulli-sigmoid model with uniform unobservables. First, we demonstrated applicability to higher-dimensional features. We show results in Figure \ref{fig:quadruple_features}. Even after quadrupling the number of features (which increases the runtime by a factor of three), both constraints still improve precision and accuracy. Secondly, we evaluate a more complex model with pairwise nonlinear interactions between features. We show results in Figure \ref{fig:pairwise_interactions}. Again both constraints generally improve precision and accuracy. We note our implementation relies on MCMC which is known to be less scalable than approaches like variational inference~\citep{wainwright2008graphical} and would likely not scale to very high-dimensional features. However, our approach does not intrinsically rely on MCMC, and incorporating more scalable estimation methods is a natural direction for future work.\footnote{We use the same simulation parameters as our standard uniform model experiments. We set the expertise constraint to apply to a random subset of 60\% of the features to match the standard uniform model experiments where expertise is assumed for 3 out of the 5 features.}
\section{UK Biobank data} 
\label{sec:appendix_ukbiobank}
\paragraph{Label processing:} In the UK Biobank (UKBB), each person's data is collected at their baseline visit. The time period we study is the 10 years preceding each person's baseline visit. $T_i \in \{0, 1\}$ denotes whether the person receives a mammogram in the 10 year period. $Y_i \in \{0, 1\}$ denotes whether the person receives a breast cancer diagnosis in the 10 year period. We verify that very few people in the dataset have $T = 0$ and $Y = 1$ (i.e., are diagnosed with no record of a test): $p(Y=1|T=0) = 0.0005$. We group these people with the untested $T=0$ population, since they did not receive a breast cancer test.
\begin{figure}
\vspace{-2em}
  \begin{center}
    \includegraphics[width=0.9\textwidth]{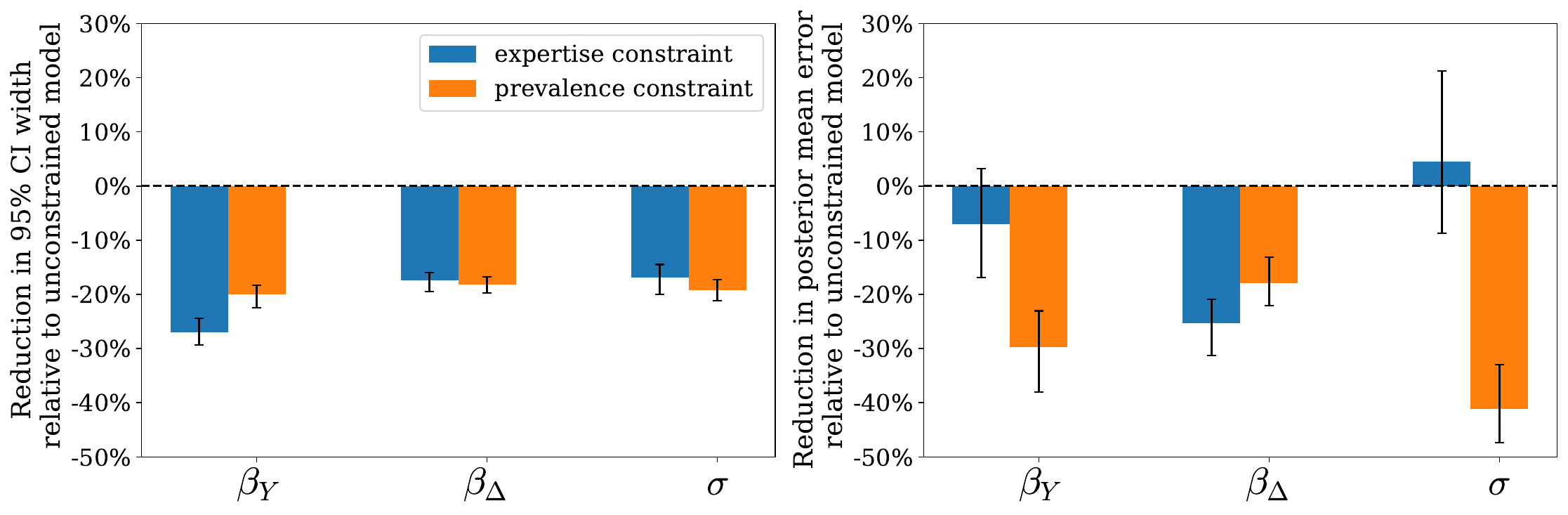}
  \end{center}
  \caption{The prevalence and expertise constraints still improve parameter inference even when using pairwise nonlinear interactions between features (rather than only linear terms, as shown in Figure \ref{fig:bar_plot}). Results are shown using synthetic data from the Bernoulli-sigmoid model with uniform unobservables. Both constraints generally produce more precise and accurate inferences on this synthetic data. We plot the median across 200 synthetic datasets. Errorbars denote the bootstrapped 95\% confidence interval on the median.}
  \label{fig:pairwise_interactions}
  \vspace{-1.5em}
\end{figure}
\paragraph{Feature processing:} 
We include features which satisfy two desiderata. First, we use features that previous work has found to be predictive of breast cancer~\citep{nih_risk_tool, cancer_risk_predictors2,yanes2020clinical}. Second, since features are designed to be used in predicting $T_i$ and $Y_i$, they must be measured prior to $T_i$ and $Y_i$ (i.e., at the beginning of the 10 year study period). Since the start of our 10 year study period occurs before the date of data collection, we choose features that are either largely time invariant (e.g. polygenic risk score) or that can be recalculated at different points in time (e.g. age). The full list of features that we include is: breast cancer polygenic risk score, previous biopsy procedure (based on OPCS4 operation codes), age at first period (menarche), height, Townsend deprivation index\footnote{The Townsend deprivation index is a measure of material deprivation that incorporates unemployment, non-car ownership, non-home ownership, and household overcrowding~\citep{townsend1988health}.}, race (White, Black/mixed Black, and Asian/mixed Asian), and age at the beginning of the study period ($<$35, 35-39, and 40-45). We normalize all features to have mean 0 and standard deviation 1.
\paragraph{Sample filtering:} We filtered our sample based on four conditions. (i) We removed everyone without data on whether or not they received breast cancer testing, which automatically removed all men because UKBB does not have any recorded data on breast cancer tests for men. (ii) We removed everyone who was missing data (e.g. responded ``do not know'') for breast cancer polygenic risk score; previous biopsy procedure; menarche; height; Townsend deprivation index; race; age; duration of moderate physical activity; cooked, salad, and raw vegetable intake; weight; use of the following medication: aspirin, ibuprofen, celebrex, and naproxen; family history of breast cancer; and previous detection of carcinoma in breast. (iii) We removed everyone who did not self report being of White, Black/mixed Black, or Asian/mixed Asian race. (iv) We remove patients who were diagnosed with breast cancer before the start of our 10 year study period, as is standard in previous work~\citep{zink2023race}. (v) We removed everyone above the age of 45 at the beginning of the observation period, since the purpose of our case study is to assess how the model performs in the presence of the distribution shift induced by the fact that young women tested for breast cancer are non-representative.\footnote{To confirm that our predictive performance remains good when looking at patients of all ages, we conduct an additional analysis fitting our model on a dataset without the age filter, but keeping the other filters. (For computational tractability, we downsample this dataset to approximately match the size of the original age-filtered dataset.) We fit this dataset using the same model as that used in our main analyses, but add features to capture the additional age categories (the full list of age categories are: $<$35, 35-39, 40-44, 45-49, 50-54, $\geq$55). We find that if anything, predictive performance when using the full cohort is better than when using only the younger cohort from our main analyses in \S \ref{sec:validating_model}. Specifically, the model’s quintile ratio is 4.6 among the tested population ($T_i=1$) and 7.0 among the untested population ($T_i=0$) that attended a follow-up visit.}
\paragraph{Model fitting:} We divide the data into train and test sets with a 70-30 split. We use the train set to fit our model. We use the test set to validate our risk predictions on the tested population ($T=1$). We validate our risk predictions for the $T=1$ population on a test set because the model is provided both $Y$ and $X$ for the train set, so using a test set replicates standard machine learning practice. We do not run the other validations (predicting risk among the $T=0$ population and inference of unobservables) on a test set because in all these cases the target variable is unseen by the model during training. Overfitting concerns are minimal because we use a large dataset and few features.
\paragraph{Inferred risk predicts breast cancer diagnoses among the untested population:} When verifying that inferred risk predicts future cancer diagnoses for the people who were untested ($T_i=0$) at the baseline, we use data from the three UKBB follow-up visits. We only consider the subset of people who attended at least one of the follow-up visits. We mark a person as having a future breast cancer diagnosis if they report receiving a breast cancer diagnosis at a date after their baseline visit. 
\paragraph{Inferred unobservables correlate with known unobservables:} We verify that across people, our inferred posterior mean of unobservables correlates with a true unobservable---whether the person has a family history of breast cancer. We define a family history of breast cancer as either the person's mother or sisters having breast cancer. We do not include this data as a feature because we cannot be sure that the measurement of family history precedes the measurement of $T_i$ and $Y_i$. This allows us to hold out this feature as a validation.
\paragraph{IRB:} Our institution's IRB determined that our research did not meet the regulatory definition of human subjects research. Therefore, no IRB approval or exemption was required.
\section{Additional experiments on cancer data}
Here we provide additional sets of experiments. We provide a comparison to various baseline models (Appendix \ref{sec:baselines}) and robustness experiments (Appendix \ref{sec:appendix_robustness_experiments}).
\subsection{Comparison to baseline models}
We provide comparisons to three different types of baseline models: (i) a model trained solely on the tested population, (ii) a model which assumes the untested group is negative, and (iii) other selective labels baselines. 
\paragraph{Comparison to models trained solely on the tested population:}
\label{sec:baselines}

The first baseline that we consider is a model which estimates $p(Y_i=1|T_i=1, X_i)$: i.e., a model which predicts outcomes without unobservables using only the tested population.\footnote{We estimate this using a logistic regression model, which is linear in the features. To confirm that non-linear methods yield similar results, we also fit random forest and gradient boosting classifiers. These methods achieve similar predictive performance to the linear model and they also predict an implausible age trend.} This is a widely used approach in medicine and other selective labels settings. In medicine, it has been used to predict COVID-19 test results among people who were tested~\citep{jehi2020individualizing,mcdonald2021derivation}; to predict hypertrophic cardiomyopathy among people who received gold-standard imaging tests~\citep{farahani2020explanatory}; and to predict discharge outcomes among people deemed ready for ICU discharge~\citep{mcwilliams2019towards}. It has also been used in the settings of policing~\citep{lakkaraju2017selective}, government inspections~\citep{lauferend}, and lending~\citep{bjorkegren2020behavior}. 

As shown in Figure \ref{fig:baselines}, we find that the model trained solely on the tested population learns that cancer risk first increases with age and then falls sharply, contradicting prior epidemiological and physiological evidence~\citep{cancer_risk_predictors2, cr_prevalence_stats,us2013us,campisi2013aging}. We see this same trend for a model fit without a prevalence constraint in \S \ref{sec:comparison_without_prevalence}. This indicates that these models do not predict plausible inferences consistent with prior work.
\begin{wrapfigure}{r}{0.5\textwidth}
\vspace{1.5em}
  \begin{center}
    \includegraphics[width=0.48\textwidth]{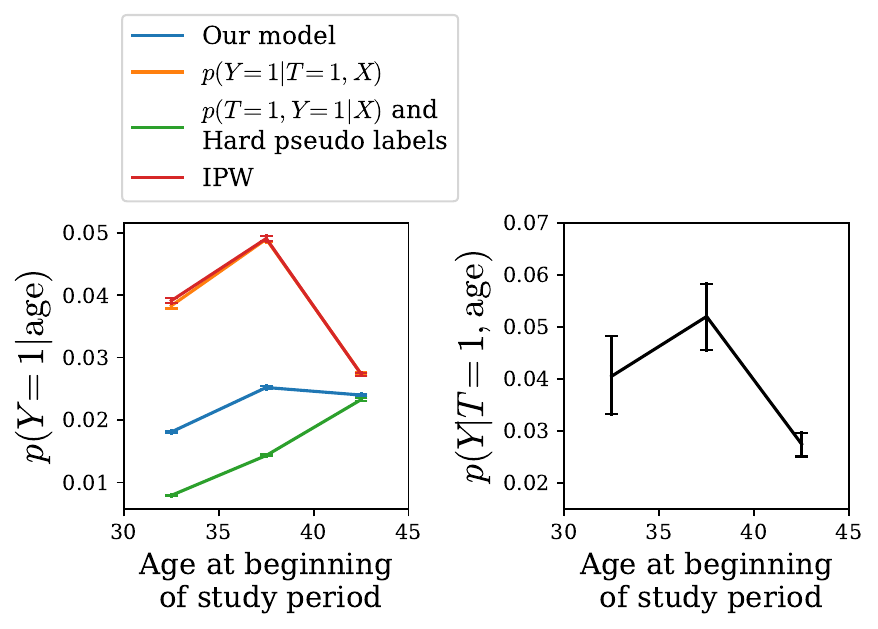}
  \end{center}
  \caption{We run three sets of baseline models: (i) models trained solely on the tested population, estimating $p(Y_i=1|T_i=1, X_i)$; (ii) models which treat the untested group as negative, estimating $p(T_i=1, Y_i=1|X_i)$; and (iii) other selective labels baselines (IPW and hard pseudo labels). Both IPW and the model estimating $p(Y_i=1|T_i=1, X_i)$ learn that cancer risk first increases and then decreases with age, contradicting prior literature. This implausible inference occurs because the tested population has the same misleading age trend (right plot). In contrast, our Bayesian model learns a more plausible age trend (left plot, blue line). Hard pseudo labels and the model estimating $p(T_i=1, Y_i=1|X_i)$ also learn plausible age trends, but they underperform our Bayesian model in predictive performance.}
  \label{fig:baselines}
  \vspace{-4em}
\end{wrapfigure}
\paragraph{Comparison to a model which treats the untested group as negative:}
We also consider a baseline model which treats the untested group as negative; this is equivalent to predicting $p(T_i=1, Y_i=1|X_i)$, an approach used in prior selective labels work \citep{shen2021artificial, ko2020detection, rastogi2023learn}. We find that, though this baseline no longer learns an implausible age trend, it underperforms our model in terms of AUC (AUC is 0.60 on the tested population vs. 0.63 for our model; AUC is 0.60 on the untested population vs. 0.63 for our model) and quintile ratio (quintile ratio on the tested population is 2.4 vs. 3.3 for our model; quintile ratio for both models is 2.5 on the untested population). This baseline is a special case of our model with the prevalence constraint set to $p(Y=1|T=0) = 0$, an implausibly low prevalence constraint. In light of this, it makes sense that this baseline learns a more plausible age trend, but underperforms our model overall. 

\paragraph{Comparison to other selective labels baselines:}
We also consider two other common selective labels baselines \citep{rastogi2023learn}. First, we predict hard pseudo labels for the untested population \citep{Lee2013PseudoLabelT}: i.e., we train a classifier on the tested population and use its outputs as pseudo labels for the untested population. Due to the low prevalence of breast cancer in our dataset, the pseudo labels are all $Y_i=0$, so this model is equivalent to treating the untested group as negative and similarly underperforms our model in predictive performance. Second, we use inverse propensity weighting (IPW) \citep{shimodaira2000improving}: i.e., we train a classifier on the tested population but reweight each sample by the inverse propensity weight $\frac{1}{p(T_i=1|X_i)}$.\footnote{We clip $p(T_i=1|X_i)$  to be between [0.05, 0.95], consistent with previous work.} As shown in Figure \ref{fig:baselines}, this baseline also learns the implausible age trend that cancer risk first increases and then decreases with age: this is because merely reweighting the sample, without encoding that the untested patients are less likely to have cancer via a prevalence constraint, is insufficient to correct the misleading age trend.

\subsection{Robustness checks for the breast cancer case study}
\label{sec:appendix_robustness_experiments}
Our primary breast cancer results (\S \ref{sec:real_data_experiments}) are computed using the Bernoulli-sigmoid model in \eqref{eq:uniform_model}. In this model, unobservables are drawn from a uniform distribution, $\alpha$ is set to $1$, and the prevalence constraint is set to $p(Y=1) = 0.02$ based on previously reported breast cancer incidence statistics~\citep{cr_prevalence_stats}. In order to assess the robustness of our results, we show that they remain consistent when altering all three of these aspects to plausible alternative specifications.
\begin{figure}[]
\vspace{-2em}
     \centering
     \begin{subfigure}[b]{0.59\textwidth}
         \centering
         \includegraphics[height=0.35\textheight]{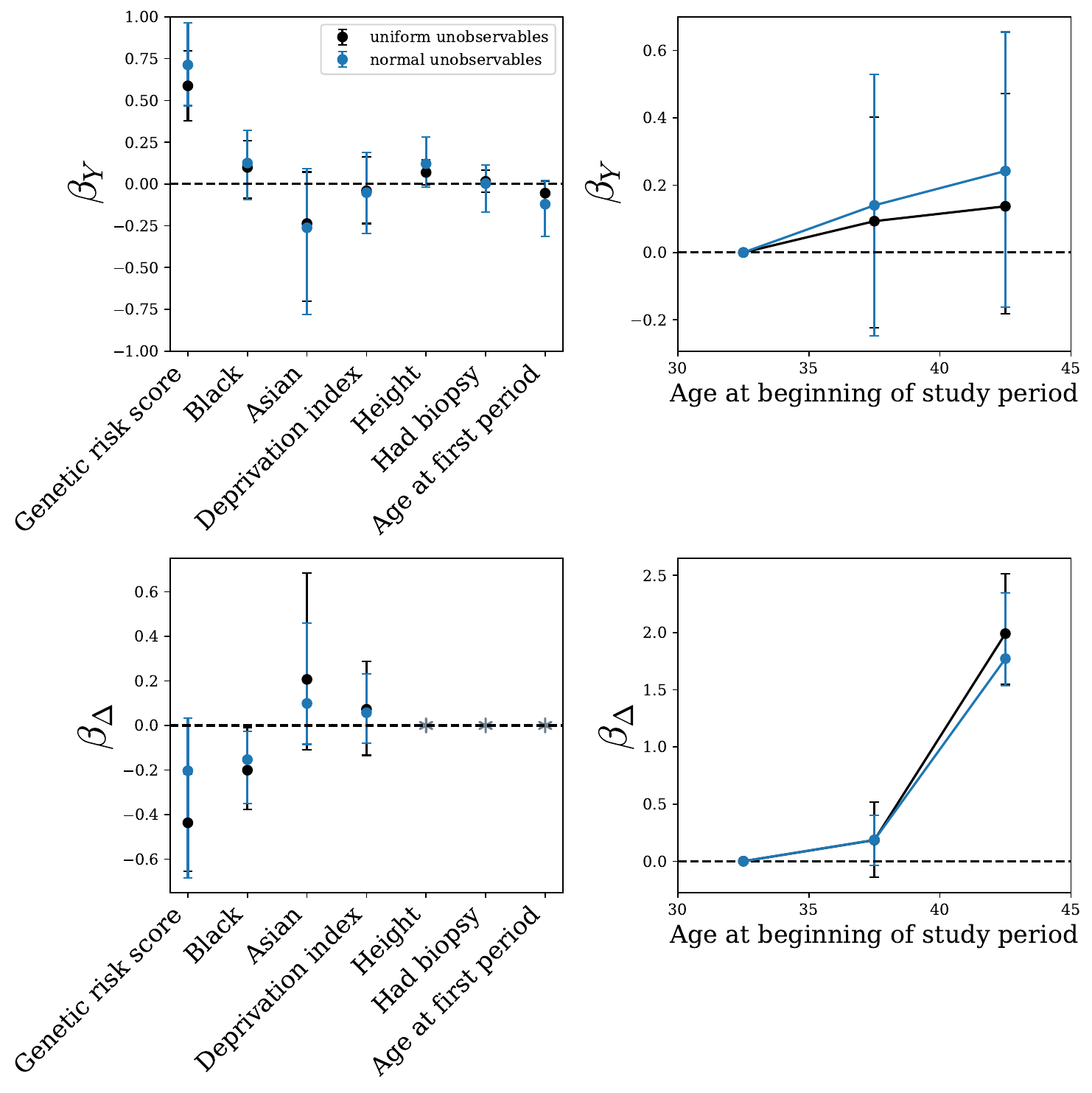}
         \caption{}
         \label{fig:noise_robustness}
     \end{subfigure}
     \hfill
    \begin{subfigure}[b]{0.4\textwidth}
         \centering
         \includegraphics[height=0.35\textheight]{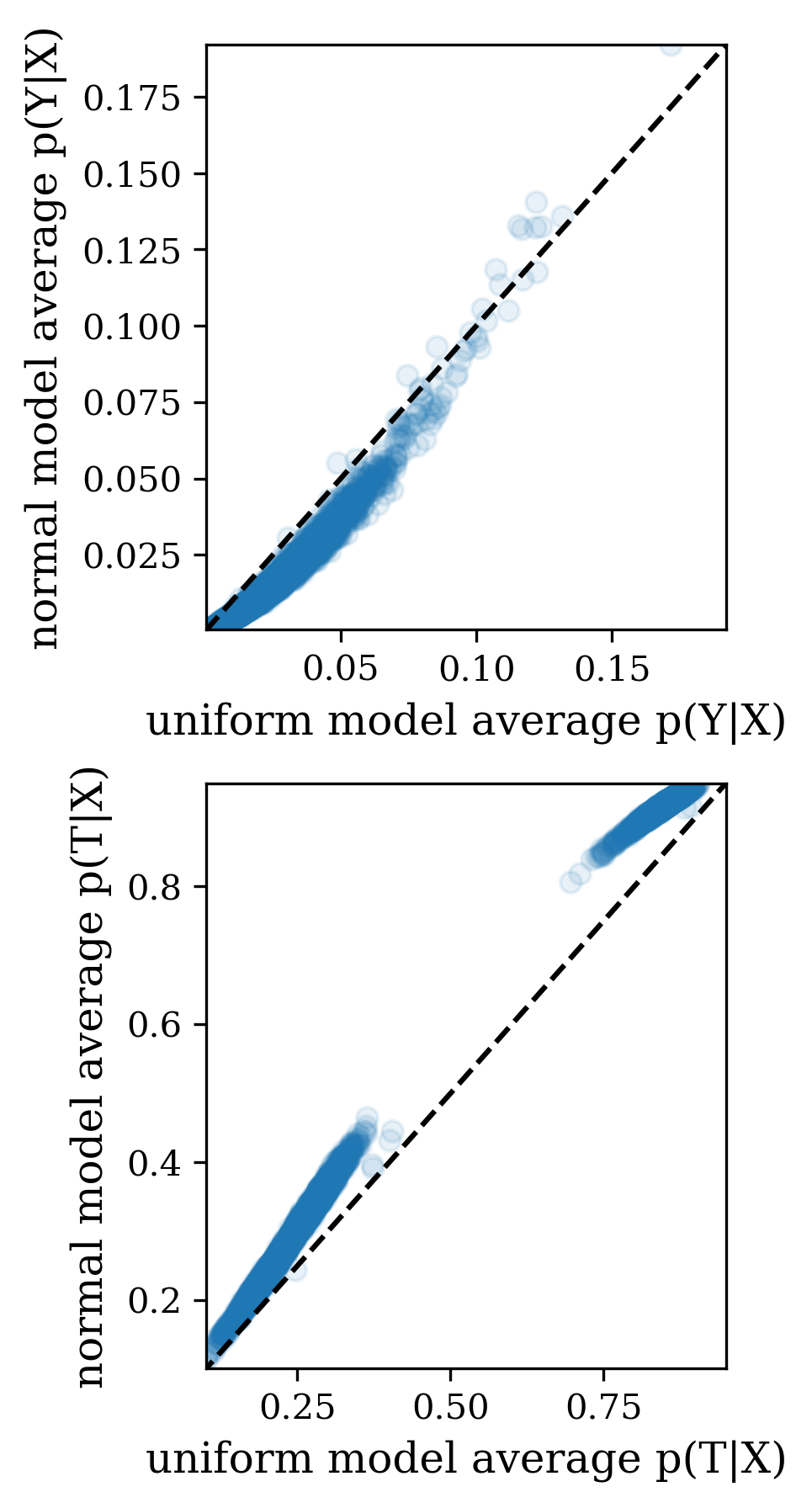}
         \caption{}
         \label{fig:noise_comparison}
     \end{subfigure}
     \caption{We compare the results from the uniform unobservable model in \eqref{eq:uniform_model} (black) and the normal unobservable model in \eqref{eq:normal_model} (blue). Figure \ref{fig:noise_robustness}: The estimated $\betaY$ and $\betadelta$ coefficients remain similar for both models, with similar trends in the point estimates and overlapping confidence intervals. Figure \ref{fig:noise_comparison}: Both models predict highly correlated values for $p(Y_i|X_i)$ and $p(T_i|X_i)$. Perfect correlation is represented by the dashed line.}
     \label{fig:noise}
     \vspace{-1em}
\end{figure}
\begin{figure}
\vspace{-2em}
     \centering
     \begin{subfigure}[b]{0.59\textwidth}
         \centering
         \includegraphics[height=0.35\textheight]{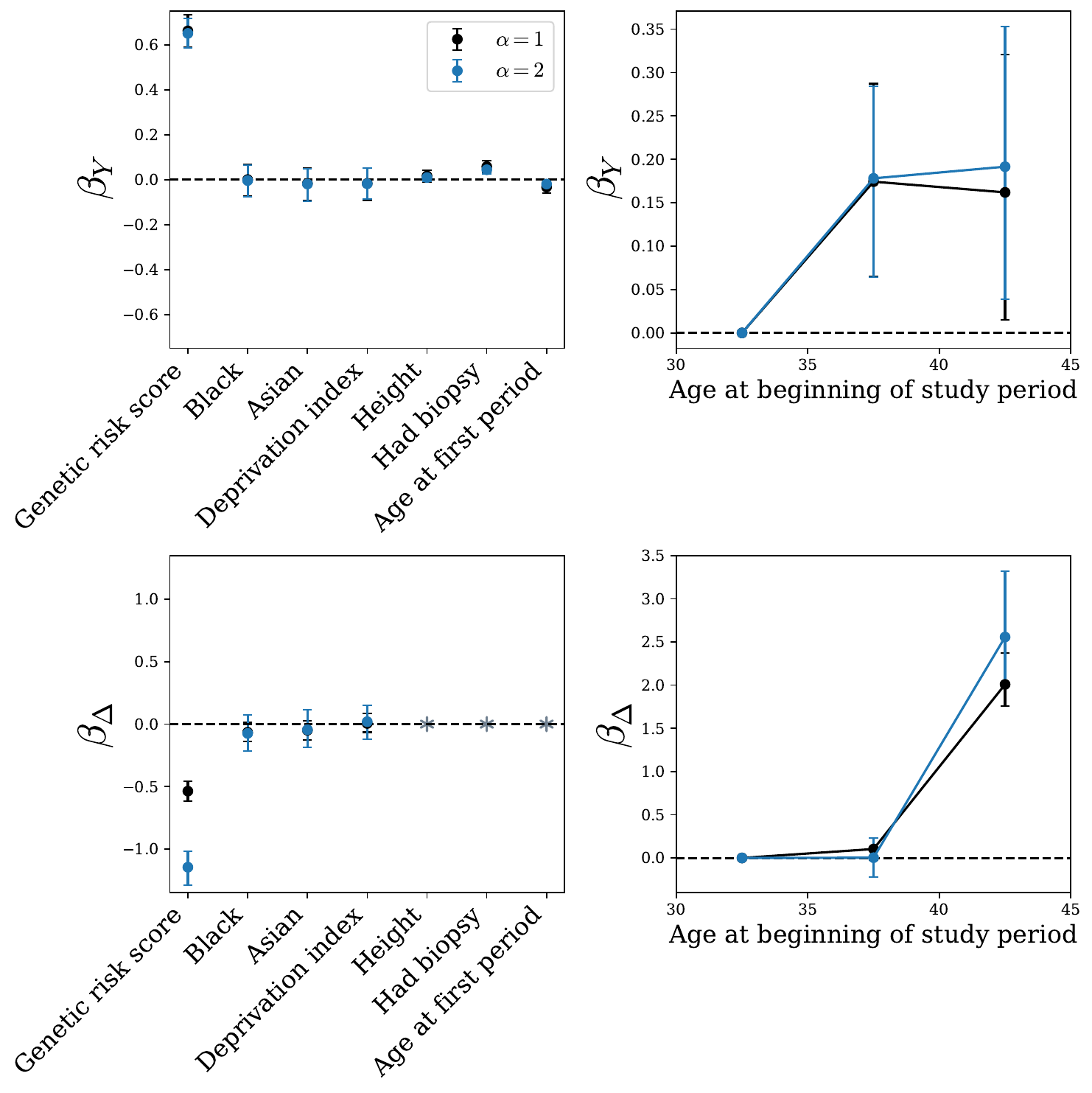}
         \caption{}
         \label{fig:alpha_robustness}
     \end{subfigure}
     \hfill
    \begin{subfigure}[b]{0.4\textwidth}
         \centering
         \includegraphics[height=0.35\textheight]{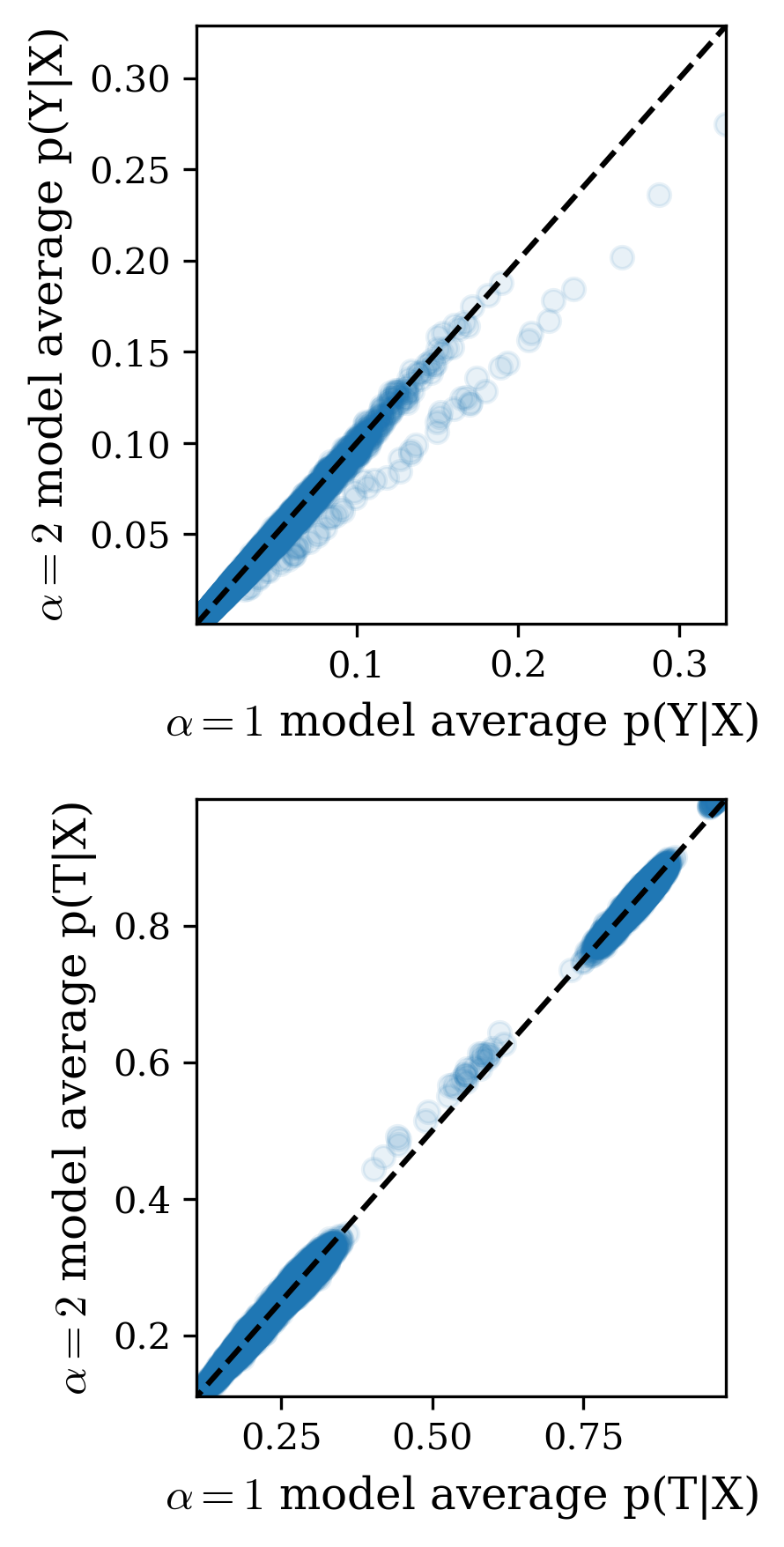}
         \caption{}
         \label{fig:alpha_comparison}
     \end{subfigure}
     \caption{We compare the results from the uniform unobservable model with $\alpha=1$ (black) and $\alpha=2$ (blue). Figure \ref{fig:alpha_robustness}: The inferred $\betaY$ and $\betadelta$ coefficients are generally very similar, with similar trends in the point estimates and overlapping confidence intervals. The only exception is the estimate of $\betadelta$ for genetic risk, which is explained by the fact that the prediction of $\betadelta$ depends on the value of $\alpha$. Figure \ref{fig:alpha_comparison}: Both models predict highly correlated values for $p(Y_i|X_i)$ and $p(T_i|X_i)$. Perfect correlation is represented by the dashed line.}
     \label{fig:alpha}
     \vspace{-0em}
\end{figure}
\paragraph{Consistency across different distributions of unobservables:}
We compare the uniform unobservables model (\eqref{eq:uniform_model}) to the normal unobservables model (\eqref{eq:normal_model}). As described in Appendix \ref{sec:appendix_syn_experiments}, the normal unobservables model does not allow us to marginalize out $Z_i$ and thus converges more slowly. Hence, for computational tractability, we run the model on a random subset of $\frac{1}{8}$ of the full dataset. In Figure \ref{fig:noise_robustness}, we see that the estimated coefficients for both models remain similar, with similar trends in the point estimates and overlapping confidence intervals. Figure \ref{fig:noise_comparison} shows that the inferred values of $p(Y_i|X_i)$ and $p(T_i|X_i)$ for each data point also remain correlated, indicating that the models infer similar testing probabilities and disease risks for each person.
\begin{figure}
\vspace{-2em}
     \centering
     \begin{subfigure}[b]{0.59\textwidth}
         \centering
         \includegraphics[height=0.35\textheight]{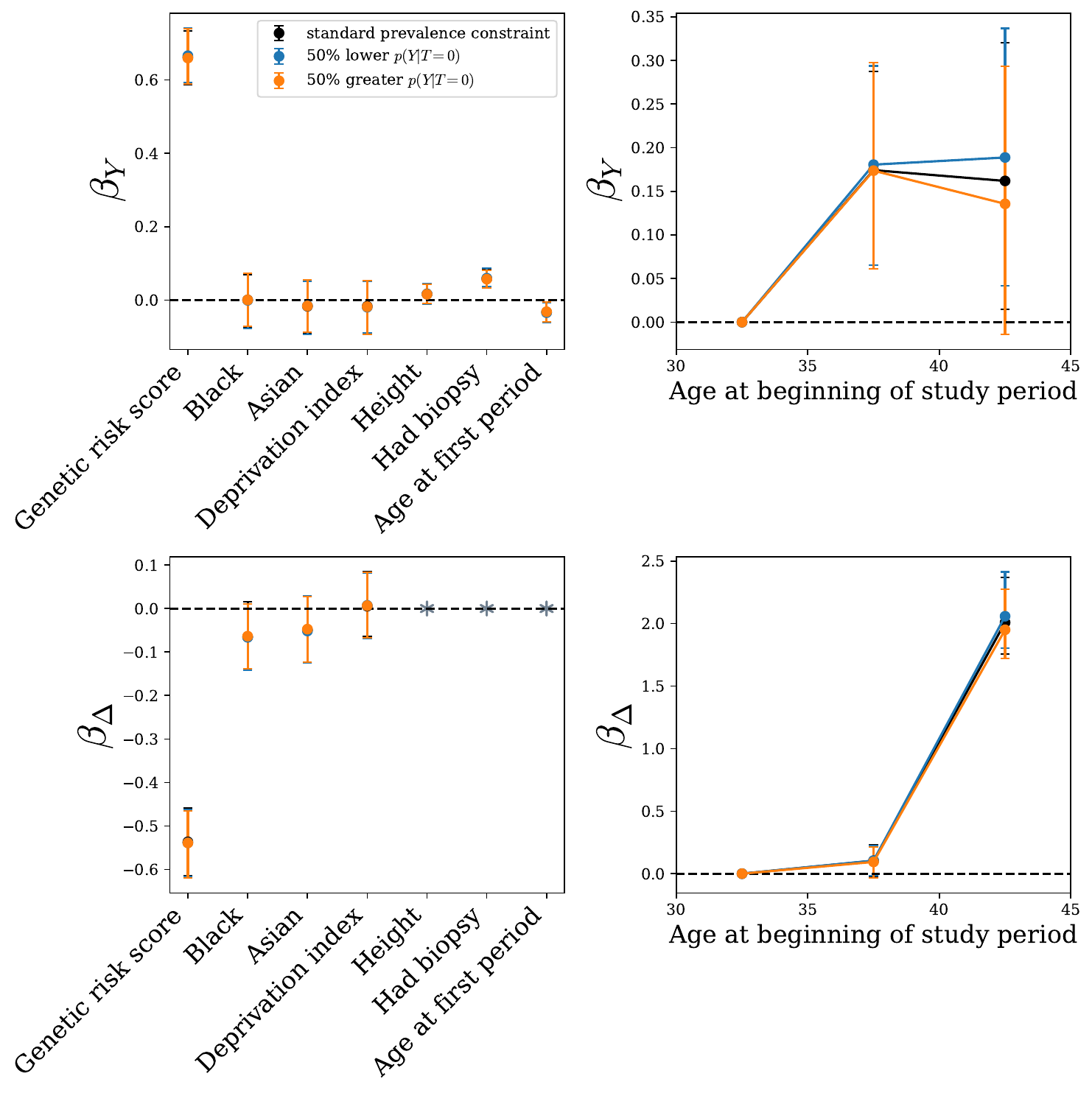}
         \caption{}
         \label{fig:prevalence_robustness}
     \end{subfigure}
     \hfill
    \begin{subfigure}[b]{0.4\textwidth}
         \centering
         \includegraphics[height=0.35\textheight]{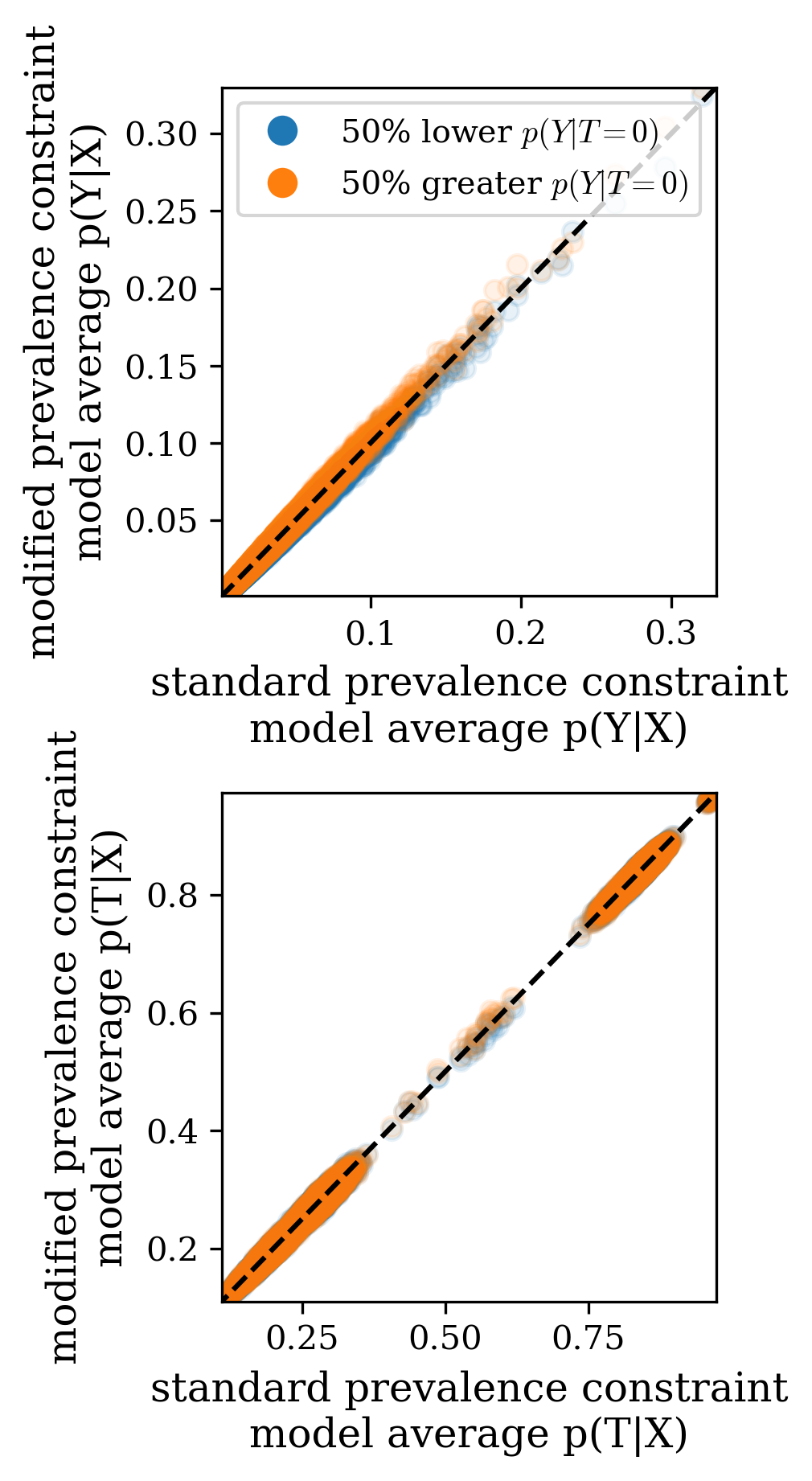}
         \caption{}
         \label{fig:prevalence_comparison}
     \end{subfigure}
     \caption{We compare the results from the uniform unobservables model with a prevalence constraint of $\mathbb{E}[Y]=0.02$ informed by cancer statistics~\citep{cr_prevalence_stats} (black), a prevalence constraint which corresponds to 50\% less of the untested population having the disease (blue), and a prevalence constraint which corresponds to 50\% more of the untested population having the disease (orange). Figure \ref{fig:prevalence_robustness}: The predictions for all three models are similar as seen by the similar trends in the point estimates and overlapping confidence intervals. Figure \ref{fig:prevalence_comparison}: All three models predict correlated values for $p(Y_i|X_i)$ and $p(T_i|X_i)$. Perfect correlation is represented by the dashed line.}
     \label{fig:prevalence}
     \vspace{-2em}
\end{figure}
\paragraph{Consistency across different $\alpha$:}
We compare the uniform unobservables model with $\alpha=1$ to a uniform unobservables model with $\alpha=2$. In Figure \ref{fig:alpha_robustness}, we see that the inferred coefficients for both models are generally very similar, with similar trends in the point estimates and overlapping confidence intervals. The only exception is $\betadelta$ for the genetic risk score. While both models find a negative $\betadelta$ for the genetic risk score, indicating genetic information is underused, the coefficient is less negative when $\alpha = 1$. This difference occurs because altering $\alpha$ changes the assumed relationship between the risk score and the testing probability under purely risk-based allocation, and thus changes the estimated deviations from this relationship (which $\betadelta$ captures). Past work also makes assumptions about the relationship between risk and human decision-making~\citep{pierson2020assessing,simoiu2017problem,pierson2018fast,pierson2020large}. We can restrict the plausible values of $\alpha$, and thus $\betadelta$, using the following approaches: (i) restricting $\alpha$ to a range of reasonable values based on domain knowledge; (ii) setting $\alpha$ to the value predicted by a model with $\sigma^2$ pinned; or (iii) fitting $\alpha$ and $\sigma^2$ in a model with non-binary $Y_i$ outcomes when both parameters can be simultaneously identified. 

To confirm model consistency, we compare the inferred values of $p(Y_i|X_i)$ and $p(T_i|X_i)$ for each data point. As shown in Figure \ref{fig:alpha_comparison}, these estimates remain highly correlated across both models, indicating that the models infer similar testing probabilities and disease risks for each person.
\paragraph{Consistency across different prevalence constraints:}
The prevalence constraint fixes the estimate of $p(Y=1)$. Because the proportion of tested individuals who have the disease, $p(Y=1|T=1)$, is known from the observed data, fixing $p(Y=1)$ is equivalent to fixing the proportion of \emph{untested} individuals with the disease, $p(Y=1|T=0)$. For the model in \S \ref{sec:real_data_experiments}, we set the prevalence constraint to $0.02$ based on cancer incidence statistics~\citep{cr_prevalence_stats}. However, disease prevalence may not be exactly known~\citep{manski2021estimating, manski2020bounding, mullahy2021embracing}. To check the robustness of our results to plausible variations in the prevalence constraint, we compare to two other prevalence constraints that correspond to 50\% lower and 50\% higher values of $p(Y=1|T=0)$.\footnote{While our results are robust to significant alterations of the prevalence constraint, we do note that if the model is run with a wildly misspecified prevalence constraint --- for example, $p(Y=1|T=0) = 0$ --- it could produce incorrect results. To avoid this issue, our Bayesian framework also accommodates approximate constraints, if the prevalence is only approximately known.} This yields overall prevalence constraints of $\mathbb{E}[Y]\approx0.018$ and $0.022$, respectively.
In Figure \ref{fig:prevalence_robustness}, we compare the $\betaY$ and $\betadelta$ coefficients for these three different prevalence constraints. Across all three models, the estimated coefficients remain similar, with similar trends in the point estimates and overlapping confidence intervals. In particular, the age trends also remain similar in all three models, in contrast to the model fit without a prevalence constraint (\S \ref{sec:comparison_without_prevalence}). In Figure \ref{fig:prevalence_comparison}, we compare the inferred values of $p(Y_i|X_i)$ and $p(T_i|X_i)$ for each data point and confirm that these estimates remain highly correlated across all three models, indicating that the models infer very similar testing probabilities and disease risks for each person.

\end{document}